\documentclass[11pt]{article}

\usepackage[margin=1in]{geometry}
\usepackage{graphicx}
\graphicspath{{./}}
\usepackage{multirow}
\usepackage{amsmath,amssymb,amsfonts}
\usepackage{amsthm}
\usepackage{mathrsfs}
\usepackage[title]{appendix}
\usepackage{xcolor}
\usepackage{textcomp}
\usepackage{caption}
\usepackage{booktabs,longtable,array}
\usepackage{fvextra}
\usepackage{needspace}
\usepackage{placeins}
\newcounter{maintablecount}
\usepackage{algorithm}
\usepackage{algorithmicx}
\usepackage{algpseudocode}
\usepackage[numbers,sort&compress]{natbib}
\usepackage[hidelinks,breaklinks=true]{hyperref}

\hypersetup{
    pdftitle={MARGIN: Runtime Confidence Calibration for Multi-Agent Foundation Model Coordination},
    pdfauthor={Joss Armstrong},
    pdfsubject={Online confidence calibration for foundation model multi-agent systems},
    pdfkeywords={confidence calibration, multi-agent systems, foundation models, online learning, distribution shift}
}

\theoremstyle{plain}
\newtheorem{theorem}{Theorem}

\theoremstyle{definition}

\title{MARGIN: Runtime Confidence Calibration for Multi-Agent\\Foundation Model Coordination}

\author{%
Joss Armstrong\\
Ericsson, Athlone, Ireland\\
\texttt{joss.armstrong@ericsson.com}
}

\date{8 October 2026}

\begin{document}
\maketitle

\begin{abstract}
When a coordinator compares answers from heterogeneous foundation models, self-reported confidence may have different meanings across responders and changing workloads. This paper presents MARGIN (Multi-Agent Runtime Grading via Incremental Normalisation), a runtime calibration method that learns model-specific confidence corrections from observed answer outcomes without retraining the models or requiring a held-out calibration set. MARGIN tracks recent accuracy and stated confidence within confidence bands, uses their ratio to correct reported confidence, and blends sparse-band corrections toward a model-level estimate. The corrected scores weight candidate answers in a collective decision. Evaluation covers code generation, question answering, and mathematics, using an 18-model pool and a nine-model subset for distribution-shift experiments. On BigCodeBench, model-mean confidence is negatively related to accuracy; among correct/incorrect response pairs, choosing the more confident responder performs below chance. Against five online calibration baselines receiving identical feedback and retaining their learned state across each transition, MARGIN achieves lower post-shift expected calibration error than all five in two code-generation transitions and than four in a question-answering transition; the remaining question-answering comparison is inconclusive. In separate code-generation coordination experiments, calibration improves the ranking of correct responses and increases answer-selection accuracy by 4.3 and 14.0 percentage points on two of three benchmarks relative to uncalibrated confidence weighting. These results support model-specific runtime calibration for coordination under changing workloads when correctness feedback is available for the participating responders.

\end{abstract}

\section{Introduction}\label{sec:introduction}

A coordinator combining answers from heterogeneous foundation models must decide how much weight to give each response. Self-reported confidence provides an immediate signal, but the same numerical score can mean different things for different models. For example, a responder claiming 90\% confidence may be less reliable than one claiming 70\%. Language-model confidence is often miscalibrated~\cite{xiong2024llm,geng2024survey}; when those scores are used to compare or aggregate answers, miscalibration becomes a coordination problem rather than just an individual-model diagnostic.

Changing workloads introduce a second difficulty. A confidence correction describes the relationship between stated confidence and correctness on the tasks encountered. When the task distribution changes, a correction learned from earlier outcomes may no longer be appropriate. This can happen without changing the models themselves: the questions change while their weights and response policies remain fixed. A confidence-based coordinator therefore needs both model-specific corrections and a way to revise them as new evidence arrives.

Fixed post-hoc methods, including temperature scaling, Platt scaling, and histogram binning, fit corrections on calibration data~\cite{guo2017calibration,platt1999probabilistic,naeini2015obtaining}. Their performance can deteriorate under distribution shift~\cite{ovadia2019trust}. Online calibration instead updates from subsequent outcomes~\cite{dawid1982calibrated,foster1998asymptotic}. MARGIN operates in this online setting alongside the histogram, windowed-accuracy, and online Platt methods evaluated here. The relevant comparison is between correction rules that receive the same feedback and must adapt calibration learned on an earlier workload to a changed task distribution.

MARGIN maintains exponentially weighted estimates of accuracy and stated confidence within each model's confidence bands. Their ratio discounts or increases reported confidence according to observed reliability. Separate bands allow confidence-dependent corrections; within-model shrinkage shares information when a band has few observations; and exponential weighting reduces the influence of older outcomes. The calibrated scores then weight candidate answers in a collective vote. The construction uses only model outputs, stated confidence, and subsequent correctness feedback, without model internals, retraining, or a separate calibration dataset.

The evaluation separates adaptation after a workload change from calibration started without prior history. For each shift, all online methods receive the same source observations and outcomes, retain their state at the transition, and continue updating on the target workload. MARGIN has lower post-shift expected calibration error (ECE) than all five online baselines in two code-generation transitions and than four of five in a question-answering transition. The remaining question-answering comparison is inconclusive. Mathematics and same-task test-hardening conditions complete the shift evaluation. Single-benchmark experiments separately assess performance without prior calibration history, while a selected-model-feedback diagnostic tests the dependence on observing outcomes for the whole responding pool.

Separate coordination experiments establish the downstream relevance of calibration. MARGIN improves the ranking of correct responses on the three evaluated code-generation benchmarks and answer-selection accuracy on two. Previous evaluations have found that verbalized confidence can poorly distinguish correct from incorrect answers, including below-chance discrimination in some model--task settings~\cite{xiong2024llm}. Our BigCodeBench evaluation identifies confidence inversion across the heterogeneous pool: model-mean confidence is negatively associated with accuracy, and among pairs containing one correct and one incorrect response to the same task, choosing the more confident responder performs below random choice before calibration (Section~\ref{sec:inversion}). The matched-feedback shift experiments then assess how alternative online corrections adapt when the workload changes.

The contributions are: (i) a confidence-band runtime calibration mechanism for heterogeneous model coordination; (ii) comparative evidence on adaptation after task-distribution changes, with online alternatives given the same feedback and retained histories; and (iii) an empirical diagnosis of confidence inversion on BigCodeBench, with evaluations of pairwise recovery, answer selection, and feedback requirements. A stationary estimator analysis explains the bias introduced by asymmetric success/failure updates. Calibration and coordination performance are established by the experiments.

\section{Problem Formulation}\label{sec:problem}

Consider a pool of foundation-model responders $\mathcal A=\{a_1,\ldots,a_N\}$ and a sequence of tasks $q_t$. For each task, a participating subset $\mathcal A_t$ produces answers $\hat y_{i,t}$ and stated confidences $c_{i,t}\in[0,1]$. A coordinator receives these outputs and selects an answer. Subsequently, a task-specific evaluator may supply the correctness outcome $o_{i,t}\in\{0,1\}$. For code generation, executable tests determine this outcome, so a correct program need not match a reference solution.

The coordinator learns a time-dependent correction $f_{i,t}$ for each responder. Its calibration objective is
\begin{equation}\label{eq:calibration-goal}
 \Pr(o_{i,t}=1\mid\tilde c_{i,t}=p)\approx p,
 \qquad \tilde c_{i,t}=f_{i,t}(c_{i,t}),
\end{equation}
where $f_{i,t}$ uses only outcomes available before scoring task $t$. The objective is assessed on the observed stream; it is not an assumption that changing workloads are already calibrated. The corrected confidence also supplies a weight for answer-level aggregation.

This is a coordinator/responder setting within multi-agent systems. An agent here is a model instance participating in the pool. The evaluated responders independently answer each task, retain no inter-task state, and do not receive the coordinator's calibration state as feedback. MARGIN changes only how the coordinator interprets their confidence, and their response policies stay fixed. No model weights, logits, or training data are available to the coordinator, and no held-out calibration set is used.

The primary experiments provide correctness feedback for every usable response after scoring. Section~\ref{sec:feedback} separately examines a diagnostic in which only the responder with the highest calibrated confidence receives an update. Interactive deliberation, planning, strategic adaptation, and trajectory-dependent confidence are outside the evaluated setting.

\section{Related Work}\label{sec:related}
\subsection{Calibration and Online Adaptation}\label{sec:rw-calibration}

Post-hoc calibration fits a mapping from prediction scores to empirical correctness. Temperature scaling adjusts the sharpness of a model distribution, Platt scaling fits a logistic correction, and histogram methods estimate correctness within score regions~\cite{guo2017calibration,platt1999probabilistic,naeini2015obtaining}. When these corrections are fitted once and held fixed, their validity depends on the relationship between the calibration and deployment distributions. Ovadia et al.~\cite{ovadia2019trust} demonstrate the importance of evaluating uncertainty under dataset shift rather than extrapolating from in-distribution calibration.

Large language model (LLM) calibration also addresses transfer and the cost of obtaining suitable supervision. Thermometer learns an auxiliary calibrator from multiple tasks for use on new tasks~\cite{shen2024thermometer}. ConfTuner fine-tunes verbalized confidence using a tokenized Brier-score objective~\cite{li2025conftuner}. Luo et al.~\cite{daca2025} use a pre-trained reference model to calibrate a post-trained model without labelled calibration data, restricting confidence alignment to examples where their predictions agree. These approaches differ in their access and training requirements. MARGIN keeps the responding models fixed and updates the coordinator's correction from observed answer outcomes.

Sequential calibration studies probabilities that are evaluated and revised as outcomes arrive~\cite{dawid1982calibrated,foster1998asymptotic}. Prediction with expert advice uses successive losses to allocate weight among predictors~\cite{cesabianchi2006prediction}; multicalibration requires calibration across specified, potentially overlapping groups~\cite{hebert2018multicalibration}. These address related but distinct objectives. MARGIN uses fixed model/confidence-band groups and evaluates both their probability correction and its consequences for answer selection. Its estimator theorem concerns update bias and makes no claim about regret or multicalibration. Adaptive conformal inference instead adapts a coverage criterion under shift~\cite{gibbs2021adaptive}, which is different from the point-confidence objective here.

The direct empirical comparison is with online histogram, windowed-accuracy, and logistic calibration rules specified in Section~\ref{sec:baselines}. They receive the same observations and retain their own source state when the workload changes. This comparison assesses the combined ratio, banding, and shrinkage construction against alternative update rules. Exponential weighting itself is established estimation machinery~\cite{hunter1986ewma,roberts1959spc}.

\subsection{Confidence Signals and Model Selection}\label{sec:rw-llm-confidence}

Confidence elicitation determines the signal available to a coordinator; calibration determines how to interpret that signal against outcomes. Tian et al.~\cite{tian2023just} study direct requests for verbalized probabilities. Xiong et al.~\cite{xiong2024llm} compare confidence-elicitation strategies and report that vanilla verbalized confidence can have poor failure-detection performance, including below-chance point estimates in some settings. Their analysis ranks different answers produced by an individual model. Our pairwise diagnostic compares responders answering the same task; a separate statistic measures the association between models' mean confidence and accuracy. The measurements are related without being interchangeable.

Other approaches derive confidence from repeated responses. Self-consistency samples reasoning paths and aggregates their answers~\cite{wang2023selfconsistency}; semantic uncertainty groups meaning-equivalent responses rather than treating every different string as a distinct answer~\cite{kuhn2023semantic,farquhar2024detecting}. Surveys organise the wider set of confidence-estimation and calibration approaches~\cite{geng2024survey,liu2025uq}. MARGIN's primary evaluation uses verbalized confidence, while matched agreement-based signals are sensitivity evidence. Its runtime correction acts on the supplied signal rather than changing the elicitation procedure.

Confidence-based model selection and cost-aware cascades make different decisions about model use. Gerych et al.~\cite{gerych2024whoknows} predict confidence for model--prompt pairs with an auxiliary regressor to select a pair for each query. FrugalGPT learns cascades that trade answer quality against model-query cost~\cite{chen2023frugalgpt}. Here the candidate responses have already been obtained: MARGIN weights their contributions to the final answer. The experiments assess this post-response coordination layer. Savings from choosing which model to query are a separate question.

\subsection{Multi-Agent Coordination and Trust}\label{sec:rw-multi-agent}\label{sec:rw-trust}

Foundation-model multi-agent systems employ debate, delegation, and structured conversation~\cite{du2023improving,wu2023autogen,zhou2025ahmad,smit2024mad}. In those settings, interaction can alter the answers that agents eventually provide. The relationship between such systems and classical agent models is discussed by La Malfa et al.~\cite{lamalfa2025llms}. MARGIN studies the confidence-bearing interface between independent responders and a coordinator. The respondents' policies remain fixed; what changes is the coordinator's interpretation of their reported confidence.

Trust and reputation methods likewise use evidence about participants to guide subsequent decisions. J{\o}sang et al.~\cite{josang2007survey} survey their varied formulations and evidence sources. EigenTrust, for example, aggregates transaction-based ratings into a global reputation value for each peer~\cite{kamvar2003eigentrust}. A single participant-level score can hide differences between confidence regions: the same responder may be reliable when reporting moderate confidence yet overstate its reliability at high confidence. MARGIN tracks this particular context explicitly through model/confidence-band state. Its multiplicative factor describes how a reported score should be adjusted, rather than serving as an unconditional ranking of agents.

This connects the method to outcome-based trust while keeping the evaluated information contract concrete: candidate answers, stated confidence, and later correctness feedback. The selected-model-feedback diagnostic tests what happens when the coordinator observes outcomes only for the responder it chooses, thereby concentrating the evidence used for future calibration.

\section{Method}\label{sec:method}

MARGIN estimates how much a responder's stated confidence should be corrected within each confidence region. Two exponentially weighted moving averages (EWMAs) track observed accuracy and stated confidence. Their ratio supplies the correction, within-model shrinkage stabilises sparsely observed bands, and the corrected scores enter an answer-level vote.

\subsection{Confidence-Band Stratified Tracking}\label{sec:bands}

MARGIN partitions $[0,1]$ into $K$ confidence bands and maintains a separate calibration factor for each (model, band) pair, allowing miscalibration to vary with stated confidence.

Archived verbalized confidences are recorded on a 0--100 scale and divided by 100 before calibration; all equations below use $c_{i,t}\in[0,1]$. The primary analysis uses $K = 3$ equal-width bands:
\begin{equation}
    B_1 = [0, \tfrac{1}{3}), \quad B_2 = [\tfrac{1}{3}, \tfrac{2}{3}), \quad B_3 = [\tfrac{2}{3}, 1].
\end{equation}

Let $\kappa(c) \in \{1, \ldots, K\}$ denote the band index for confidence value $c$. For each model $a_i$ and band $k$, we maintain two running estimates:
\begin{itemize}
    \item $\hat{a}_{i,k}$: the empirical accuracy rate of model $a_i$ when its confidence falls in band $B_k$,
    \item $\bar{c}_{i,k}$: the mean confidence expressed by model $a_i$ within band $B_k$.
\end{itemize}

Both are initialised to the band midpoint $m_k$, so that the initial calibration factor $\gamma_{i,k} = \hat{a}_{i,k} / \bar{c}_{i,k} = 1$ and raw confidence passes through unchanged before any observations.

Band boundaries stay fixed as the workload changes, so a band retains the same interpretation across the source and target streams. Section~\ref{sec:ablation} evaluates the granularity choice.

\subsection{EWMA Update Mechanism}\label{sec:ewma}

Both running estimates are updated using an exponentially weighted moving average (EWMA) with a constant learning rate $\alpha \in (0, 1)$. When model $a_i$ produces a prediction at time $t$ with confidence $c_{i,t} \in B_k$ and outcome $o_{i,t}$ is subsequently observed, the updates are:
\begin{align}
    \hat{a}_{i,k} &\leftarrow (1 - \alpha)\,\hat{a}_{i,k} + \alpha\, o_{i,t}, \label{eq:ewma-accuracy}\\
    \bar{c}_{i,k} &\leftarrow (1 - \alpha)\,\bar{c}_{i,k} + \alpha\, c_{i,t}. \label{eq:ewma-confidence}
\end{align}

Within a given (model, band) update sequence, an observation $\tau$ band-updates in the past has weight $\alpha(1-\alpha)^\tau$, giving an effective memory of roughly $1/\alpha$ observations in that band. Larger $\alpha$ tracks change faster but increases stationary noise; Section~\ref{sec:abl-alpha} reports the final sensitivity.

\paragraph{Symmetric updates.}
The same learning rate is used after a correct or incorrect answer. A faster rate after failures might seem a natural way to penalise overconfidence, but it changes the quantity estimated even under stationarity. Theorem~\ref{prop:symmetric} quantifies that bias. The primary analysis uses $\alpha=0.04$; comparative sensitivity is reported in Section~\ref{sec:ablation}.

The per-band correction is the ratio of observed accuracy to stated confidence:
\begin{equation}\label{eq:gamma}
    \gamma_{i,k} = \frac{\hat{a}_{i,k}}{\bar{c}_{i,k}}.
\end{equation}
For illustration, accuracy of $0.60$ and mean stated confidence of $0.90$ give a factor of approximately $0.67$. Before shrinkage and clipping, a new confidence of $0.90$ would then be corrected to approximately $0.60$. Dividing by mean stated confidence makes the factor a correction relative to the confidence reported in that band.

\subsection{Shrinkage Blending}\label{sec:shrinkage}

Sparse (model, band) pairs can yield high-variance factors, so MARGIN blends the band-level factor toward a model-level estimate.

Let $n_{i,k}$ denote the number of observations accumulated for model $a_i$ in band $k$. MARGIN defines the within-model target as the count-weighted ratio of band-level EWMA totals,
\begin{equation}\label{eq:model-factor}
\gamma_{i,\cdot}=
\begin{cases}
\displaystyle \frac{\sum_{j=1}^{K} n_{i,j}\hat a_{i,j}}{\sum_{j=1}^{K} n_{i,j}\bar c_{i,j}}, & \sum_{j=1}^{K}n_{i,j}>0,\\[6pt]
1, & \sum_{j=1}^{K}n_{i,j}=0.
\end{cases}
\end{equation}
The effective calibration factor is then
\begin{equation}\label{eq:blending}
    \gamma_{i,k}^\text{eff} = \frac{n_{i,k}}{n_{i,k} + k_s}\,\gamma_{i,k} + \frac{k_s}{n_{i,k} + k_s}\,\gamma_{i,\cdot}.
\end{equation}
Here $k_s>0$ controls how quickly a band's own estimate replaces the model-level target. An unobserved band uses information accumulated by the same model in other bands; it does not borrow another model's calibration history. As its count grows, the band's own ratio receives greater weight. This is a credibility-weighted blend~\cite{buhlmann1967credibility}, rather than a posterior derived from an explicit prior and likelihood. Counts are cumulative, while the accuracy and confidence estimates use exponential weights.

We fix $k_s = 100$ in the primary analysis and report condition-specific sensitivity in Section~\ref{sec:ablation}.

\subsection{Calibrated Confidence Weighting}\label{sec:selection}

Given the effective calibration factor, the calibrated confidence for model $a_i$ at time $t$ is:
\begin{equation}\label{eq:calibrated-conf}
    \tilde{c}_{i,t} = \operatorname{clip}\!\left(\gamma_{i,\kappa(c_{i,t})}^\text{eff} c_{i,t},0,1\right).
\end{equation}

For multi-agent selection, we aggregate across the responding pool using confidence-weighted voting. For each candidate answer $y$ in the response set, the aggregated score is:
\begin{equation}\label{eq:selection}
    s(y) = \sum_{i \in \mathcal{A}_t} \tilde{c}_{i,t} \cdot \mathbf{1}[\hat{y}_{i,t} = y],
\end{equation}
and the selected answer is $\hat{y}_t = \arg\max_y\, s(y)$. For the code-generation coordination experiments, answer equality uses an exact-code signature: line endings are normalised, leading and trailing newline characters are stripped, and indentation and internal whitespace are preserved. Empty or missing programs do not form an answer group.

Calibrated confidence is therefore an explicit coordination weight; selection effects are evaluated empirically because correlated errors and answer aggregation fall outside the estimator theorem.

\subsection{Prediction and Outcome Events}\label{sec:method-events}

Algorithm~\ref{alg:margin} separates prediction from feedback. At task arrival, every candidate is scored from the currently available state, the answer is selected, and the confidence and band used are retained. At outcome arrival, the supplied correctness label updates those stored states. A delayed outcome therefore uses the confidence reported for its original task rather than a new response. Exact-code equality groups votes; executable-test outcomes establish code correctness.

Only determinate outcomes and valid confidence scores enter the evaluation updates. Unknown or quarantined labels are not recoded as failures. The primary setting uses $K=3$, $\alpha=0.04$, and $k_s=100$ throughout. With $k_s=0$, used only in the ablation, the effective factor is the band ratio. The numerical implementation safeguards both ratio denominators with a lower bound of $10^{-12}$. Figure~\ref{fig:method} summarises the mechanism.

The same convex EWMA recurrence accepts a bounded graded outcome in $[0,1]$, although the Bernoulli theorem below and the experiments concern binary correctness. Multiclass probability calibration would require an explicitly defined class-specific construction; it is not evaluated here.

\begin{figure}[t]
\centering
\includegraphics[width=\linewidth]{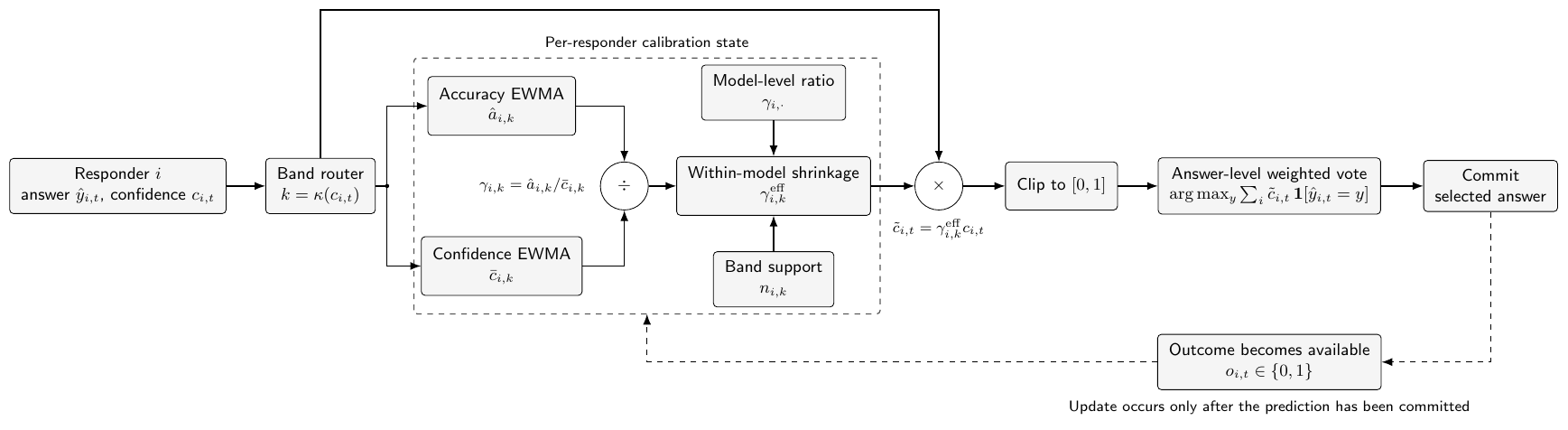}
\caption{MARGIN pipeline. For the active confidence band, MARGIN maintains separate exponentially weighted estimates of observed accuracy and stated confidence. Their ratio forms the band-level multiplicative correction, which is blended toward the model-level ratio when support is sparse. The resulting calibrated confidence weights the responder's contribution to the answer-level vote. The current outcome is incorporated only after the prediction has been scored.}\label{fig:method}
\end{figure}

\begin{algorithm}[t]
\caption{MARGIN: event-driven online calibration and multi-agent selection}\label{alg:margin}
\begin{algorithmic}[1]
\Require Agent pool $\mathcal{A}$, bands $\{B_1, \ldots, B_K\}$, learning rate $\alpha$, shrinkage $k_s$
\State \textbf{Initialise:} For all $i,k$: $\hat a_{i,k}\gets m_k$, $\bar c_{i,k}\gets m_k$, $n_{i,k}\gets0$
\Procedure{OnTask}{$q_t,\mathcal A_t$}
    \For{each responding model $a_i\in\mathcal A_t$}
        \State Receive $\hat y_{i,t}$ and stated confidence; normalise to $c_{i,t}\in[0,1]$
        \State $k\gets\kappa(c_{i,t})$
        \State Compute $\gamma_{i,k}$, $\gamma_{i,\cdot}$, and $\gamma^{\mathrm{eff}}_{i,k}$ via Eqs.~\eqref{eq:gamma}--\eqref{eq:blending}
        \State $\tilde c_{i,t}\gets\operatorname{clip}(\gamma^{\mathrm{eff}}_{i,k}c_{i,t},0,1)$
        \State Store $(i,c_{i,t},k,\hat y_{i,t})$ in pending record $P_t$
    \EndFor
    \State $\hat y_t\gets\arg\max_y\sum_{i\in\mathcal A_t}\tilde c_{i,t}\mathbf 1[\hat y_{i,t}=y]$
    \State \Return $\hat y_t$
\EndProcedure
\Procedure{OnOutcome}{$t,\{(i,o_{i,t})\}$}
    \For{each $(i,c_{i,t},k,\hat y_{i,t})\in P_t$ with supplied $o_{i,t}\in\{0,1\}$}
        \State $\hat a_{i,k}\gets(1-\alpha)\hat a_{i,k}+\alpha o_{i,t}$
        \State $\bar c_{i,k}\gets(1-\alpha)\bar c_{i,k}+\alpha c_{i,t}$
        \State $n_{i,k}\gets n_{i,k}+1$
    \EndFor
    \State Remove only the processed responders from $P_t$
\EndProcedure
\end{algorithmic}
\end{algorithm}

\section{Formal Property}\label{sec:formal}

The estimator analysis asks whether weighting failures and successes differently biases the tracked success probability. It concerns one model's accuracy estimate within one confidence band under stationary, independent outcomes. Dropping subscripts for clarity, let $X_1, X_2, \ldots$ be i.i.d.\ binary outcomes with $\mathbb{E}[X_t] = \theta$ for $0<\theta<1$, and let $\hat{a}_0 \in [0,1]$ be an arbitrary initialisation.

\begin{theorem}[Unbiasedness of the Symmetric Update]\label{prop:symmetric}
Consider an asymmetric EWMA with learning rates $\alpha_\mathrm{up},\alpha_\mathrm{down}\in(0,1)$, used when $X_t=1$ and $X_t=0$, respectively. Under i.i.d.\ $\mathrm{Bernoulli}(\theta)$ outcomes with $0<\theta<1$, the steady-state expectation is
\begin{equation}\label{eq:asymmetric-ss}
    \mathbb{E}[\hat{a}_\infty] = \frac{\alpha_\mathrm{up}\,\theta}{\alpha_\mathrm{up}\,\theta + \alpha_\mathrm{down}\,(1 - \theta)}.
\end{equation}
This equals $\theta$ if and only if $\alpha_\mathrm{up} = \alpha_\mathrm{down}$. Otherwise, the asymptotic bias is
\begin{equation}\label{eq:asymmetric-bias}
    |\mathbb{E}[\hat{a}_\infty] - \theta| = \frac{\theta(1 - \theta)\,|\alpha_\mathrm{up} - \alpha_\mathrm{down}|}{\alpha_\mathrm{up}\,\theta + \alpha_\mathrm{down}\,(1 - \theta)}.
\end{equation}
Within this two-rate EWMA family, asymptotic unbiasedness therefore holds precisely in the symmetric subfamily: any $\alpha_\mathrm{up} \neq \alpha_\mathrm{down}$ introduces the systematic bias of Eq.~\eqref{eq:asymmetric-bias}.
\end{theorem}

\noindent The proof is given in Appendix~\ref{app:proofs}. The result is confined to unbiasedness within this two-rate EWMA family under stationary i.i.d.\ Bernoulli outcomes; it does not imply MSE optimality under non-stationarity or optimality of the downstream ratio, ECE, or answer-level vote.

\section{Experimental Setup}\label{sec:setup}

\subsection{Workloads, Cohorts, and Evaluation Design}\label{sec:models}

The experiments ask two questions: how well does calibration adapt after the workload changes, and how useful are the corrected scores for choosing an answer? The first is evaluated on source-to-target transitions. The second is evaluated separately by replaying answer selection within a benchmark from an initially empty calibration history, referred to below as \emph{cold start}. A source phase provides previous calibration experience. No method is tuned on it as a held-out dataset.

Code-generation tasks provide correctness labels through execution. HumanEval (HE)~\cite{chen2021evaluating} and Mostly Basic Python Problems (MBPP)~\cite{austin2021program} supply the two source workloads; the hard subset of BigCodeBench (BCB)~\cite{zhuo2025bigcodebench} is the coding target. HumanEval+ and MBPP+ use expanded tests from EvalPlus~\cite{liu2023evalplus} on the original candidate programs. These two hardening checks change the correctness criterion for previously seen tasks; they are not independent unseen-question generalisation tests.

The question-answering transition (\emph{MMLU-shift}) moves from science, technology, engineering and mathematics (STEM) questions in Massive Multitask Language Understanding (MMLU)~\cite{hendrycks2021mmlu} to business questions in MMLU-Pro~\cite{wang2024mmlupro}. The latter introduces harder questions and up to ten answer options; benchmark, subject, and format change together. Mathematics shifts from grade-school problems in GSM8K~\cite{cobbe2021verifiers} to competition problems in MATH~\cite{hendrycks2021math}. Table~\ref{tab:audited-support} summarises retained task support. Appendix~\ref{app:protocol} gives the model inventory, prompts, scoring support, and exact comparison protocols.

The canonical pool contains 18 models\footnote{Gemma-4-31B-it was also collected but is excluded from the canonical cohort. It belongs to a later model generation than the rest of the pool, and mixing generations would conflate model capability with model era.}: nine cloud-served models from the Qwen, DeepSeek, GPT-OSS, MiniMax, and GLM families, and nine locally served Qwen, LLaMA, Phi, DeepSeek Coder, and Mistral variants. Cold-start coordination uses this pool. Shift comparisons use the same fixed nine-cloud subset with source/target coverage. Tables~\ref{supp:models}--\ref{supp:support} in Appendix~\ref{suppsec:support} give the complete model identifiers and model-level observation counts. The experiment keeps cohort, task support, confidence channel, and feedback availability matched across the methods in each comparison.

In a shift experiment, every online method receives full correctness feedback on the source workload and carries its own state into the target workload without reset. Target predictions are scored before their outcomes update any model's calibrator. Cold-start experiments instead initialise each method directly on the target benchmark. These protocols distinguish adaptation of an existing correction from estimating one without prior observations.

HumanEval and MBPP labels were re-evaluated against pinned tests. BigCodeBench was evaluated under one frozen execution configuration; four tasks whose references did not satisfy that configuration are quarantined, leaving 144 of 148 tasks. Unresolved individual labels are also excluded. HumanEval+ retains 163 target tasks and MBPP+ 224 matched target tasks. QA and mathematics use the archived answers, gold records, and deterministic scorer. That evidence authenticates the scoring, given those records, but does not extend to the complete original collection environment. Those two conditions are evaluated for calibration only, without answer-level voting.

Archived LiveCodeBench, CodeContests, TriviaQA, and original dynamic-pool experiments do not enter the retained comparisons because their historical evaluator or reference evidence was insufficiently authenticated. Appendix~\ref{suppsec:support} gives the evidence limits and retained counts.

\begin{table}[!htbp]
\caption{Workloads and retained task support. Source$\to$target counts refer to tasks. Usable model responses are counted in Table~\ref{supp:support}. ``Hardening'' reuses the original tasks under expanded tests.}\label{tab:audited-support}
\centering\small
\setlength{\tabcolsep}{4pt}
\begin{tabular}{@{}lllr@{}}
\toprule
Analysis & Interpretation & Cohort & Tasks \\
\midrule
HumanEval & Cold start & 18 models & 164 \\
MBPP & Cold start & 18 models & 257 \\
BigCodeBench & Cold start & 18 models & 144 \\
HE$\to$BCB & Coding workload change & 9 cloud & $164\to144$ \\
MBPP$\to$BCB & Coding workload change & 9 cloud & $257\to144$ \\
MMLU-shift & QA benchmark change & 9 cloud & $189\to200$ \\
GSM8K$\to$MATH & Mathematics workload change & 9 cloud & $200\to196$ \\
HE$\to$HE+ & Same-task hardening & 9 cloud & $164\to163$ \\
MBPP$\to$MBPP+ & Same-task hardening & 9 cloud & $257\to224$ \\
\bottomrule
\end{tabular}
\end{table}

\subsection{Online Comparators}\label{sec:baselines}

Raw stated confidence supplies an uncalibrated reference. The five online baselines each maintain per-model state and receive the same confidence/outcome observations as MARGIN. Sliding-window histogram predicts the observed success frequency in the current confidence bin within the model's most recent $W=200$ observations. Decayed histogram tracks success within each bin using exponential decay $0.98$. Both use ten equal-width bins, with raw confidence as the fallback before the bin has observations.

Windowed accuracy in \emph{replace} mode uses the model's success rate over its last 200 observations as the confidence score. In \emph{multiply} mode it scales the current confidence by the ratio of windowed accuracy to mean stated confidence. Online Platt scaling updates a per-model logistic map by stochastic gradient descent, with learning rate $0.1$ and $L_2$ coefficient $10^{-4}$. Appendix~\ref{suppsec:baselines} specifies initialisation and updates, including the first-observation behaviour of decayed histogram. The no-shrinkage MARGIN variant is reported separately as a method ablation.

These comparisons use the stated fixed configurations. No estimator was tuned to its optimum. Post-hoc design-time methods motivate the problem but are not used to claim an advantage unique to MARGIN. The coordination experiments use already-elicited answers from the pool; they neither compare total inference cost against a single model nor establish superiority to the best model selected in hindsight.

\subsection{Metrics, Aggregation, and Uncertainty}\label{sec:metrics}

All online scores use predict-before-update ordering. Confidence is normalised to $[0,1]$ and clipped before scoring. ECE uses ten equal-width bins, distinct from MARGIN's $K=3$ state bands:
\begin{equation}\label{eq:ece}
 \mathrm{ECE}=\sum_{b=1}^{10}\frac{n_b}{N_{\rm obs}}
       |\mathrm{acc}_b-\mathrm{conf}_b|.
\end{equation}
Here $N_{\rm obs}$ is the total usable prediction count. ECE pools responders within a cohort and is reported on a 0--100 scale; smaller values indicate a closer match between confidence and observed correctness. Rows without valid confidence or determinate correctness are excluded from these calibration calculations.

Answer-selection accuracy is the fraction of retained tasks for which the answer chosen by Eq.~\eqref{eq:selection} passes the task evaluator. If no usable answer group exists, the selector abstains and that task contributes zero. Code-group ties follow the fixed order of model identifiers, with uniform-tie expectations reported as a sensitivity analysis in Table~\ref{supp:ties}. Pairwise resolution considers only pairs containing exactly one correct response. It counts a win when that responder has higher confidence and half a win for equal confidence. Both-correct and both-wrong pairs are excluded, so 50\% is the random-choice reference.

Metrics aggregate 100 recorded shuffle histories. For ECE, each problem's bin counts, prediction sums, and outcome sums are averaged across histories; the pooled bin residuals are then evaluated. This is not the mean of 100 independently calculated ECE values. Answer-selection accuracy averages saved success contributions, while pairwise resolution divides summed win credits by summed eligible-pair counts. Appendix~\ref{suppsec:aggregation} gives the pooled estimator and exact aggregation rules.

Paired 95\% intervals use 10,000 problem-cluster bootstrap draws, with the same sampled problems across methods. They resample the retained contributions, and no calibration trajectory is replayed inside a draw. They are conditional on the recorded histories, available predictions, and retained tasks; they do not measure uncertainty over new model pools or fresh deployments. Intervals are pointwise without multiplicity adjustment, and an interval containing zero is called inconclusive rather than evidence of equivalence.

The primary ordering processes all predictions for a problem before any of that problem's updates. Historical-observation ordering preserves the archived observation-level convention, including possible within-problem updates for the full-feedback trajectory. Both are retained for the selected-model diagnostic, with all target comparisons paired within their own protocol. The parameter sensitivity uses historical-observation ordering. Appendix~\ref{suppsec:aggregation} gives the exact seeds and processing rules. This protocol distinction explains the slightly different fixed-setting ECE values in the parameter sensitivity.

\section{Results}\label{sec:results}\label{sec:results-shift}\label{sec:results-selection}

\subsection{Adaptation after Distribution Shift}\label{sec:online-family}\label{sec:shift-codegen}\label{sec:shift-qa}

The central comparison concerns calibration learned on an earlier workload. Table~\ref{tab:online-warm} compares MARGIN with the five online baselines after the two coding transitions and the MMLU-to-MMLU-Pro transition. Every method receives the same source outcomes, retains its state at the change, and continues updating on the target stream. The matched feedback isolates the comparison from differences in access to target outcomes.

On the two transitions to BigCodeBench, MARGIN has lower post-shift ECE than all five baselines. Its ECE is 12.99 after HumanEval and 11.89 after MBPP. Online Platt, the lowest-ECE alternative on these rows, reaches 15.32 and 13.42 respectively, giving MARGIN reductions of 2.33 and 1.53 percentage points. The differences from the windowed estimators are larger, with every displayed coding-comparison interval below zero.

The advantage also appears outside code generation. On the transition from MMLU STEM questions to MMLU-Pro business questions, MARGIN reaches ECE 2.85 compared with 6.87--7.37 for the two histogram and two windowed-accuracy baselines. All four paired comparisons favour MARGIN. Online Platt reaches 3.17 and its comparison with MARGIN is inconclusive. The observed adaptation result therefore extends to a question-answering benchmark transition involving changes in subject and answer format, alongside the two coding-source transitions.

\begin{table}[!htbp]
\caption{Matched-feedback online comparisons after two coding transitions and MMLU (STEM) to MMLU-Pro (business), abbreviated MMLU-shift. Each cell is MARGIN minus comparator ECE with a paired 95\% interval. All methods use the nine-cloud cohort and retain their own source state. Table~\ref{supp:full-shifts} in Appendix~\ref{suppsec:shifts} gives all six conditions.}\label{tab:online-warm}
\centering
\footnotesize
\setlength{\tabcolsep}{3pt}
\begin{tabular}{@{}lrrr@{}}
\toprule
Comparator & HE$\to$BCB & MBPP$\to$BCB & MMLU shift \\
\midrule
Decayed histogram & $-15.08$ $[-15.98,-11.84]$ & $-13.20$ $[-14.18,-10.00]$ & $-4.33$ $[-5.89,-1.45]$ \\
Online Platt & $-2.33$ $[-2.63,-1.93]$ & $-1.53$ $[-1.87,-1.07]$ & $-0.32$ $[-1.20,+0.42]$ \\
Sliding-window hist. & $-28.07$ $[-29.12,-24.69]$ & $-25.97$ $[-27.14,-22.52]$ & $-4.34$ $[-5.51,-1.56]$ \\
Windowed acc. multiply & $-28.45$ $[-29.03,-24.77]$ & $-26.57$ $[-27.21,-22.75]$ & $-4.51$ $[-5.17,-2.59]$ \\
Windowed acc. replace & $-30.52$ $[-31.21,-26.74]$ & $-28.55$ $[-29.28,-24.65]$ & $-4.01$ $[-5.38,-0.97]$ \\
\bottomrule
\end{tabular}
\end{table}

Table~\ref{tab:audited-shifts} supplies the uncalibrated reference across all six transitions. MARGIN has lower ECE than raw confidence in every row, but that comparison alone does not distinguish it from other online methods. Table~\ref{supp:full-shifts} in Appendix~\ref{suppsec:shifts} gives the complete six-condition online comparison. On the mathematics transition, the pointwise interval favours MARGIN against online Platt, with the other four comparisons inconclusive. That Platt interval ends close to zero and is reported to three decimals in Table~\ref{supp:full-shifts}. All five comparisons on each of the two same-task hardening checks are inconclusive. The comparative claim is therefore specific to the reported transitions and configurations.

\begin{table}[!htbp]
\caption{Post-transition ECE on the nine-cloud cohort. Every online method retains its source state. All deltas are MARGIN minus Raw in percentage points; negative values favour MARGIN. Intervals are paired problem-cluster 95\% intervals.}\label{tab:audited-shifts}
\centering
\small
\setlength{\tabcolsep}{3pt}
\begin{tabular}{@{}lrrrr@{}}
\toprule
Condition & Raw & MARGIN & $\Delta$ & 95\% interval \\
\midrule
HE$\to$BCB & 66.79 & 12.99 & $-53.80$ & $[-54.41,-50.16]$ \\
MBPP$\to$BCB & 66.79 & 11.89 & $-54.90$ & $[-55.55,-51.13]$ \\
MMLU-shift & 18.70 & 2.85 & $-15.85$ & $[-17.00,-12.50]$ \\
GSM8K$\to$MATH & 26.16 & 3.45 & $-22.70$ & $[-24.16,-17.17]$ \\
HE$\to$HE+ & 8.57 & 3.57 & $-5.00$ & $[-6.25,-1.99]$ \\
MBPP$\to$MBPP+ & 17.96 & 3.07 & $-14.89$ & $[-16.18,-9.82]$ \\
\bottomrule
\end{tabular}
\end{table}

\FloatBarrier
\subsection{Coordination and Answer Selection}\label{sec:coordination}

Separate cold-start experiments assess whether corrected confidence improves the coordinator's decisions, rather than only a calibration statistic. Here each method starts without source history, and the 18-model pool answers the same tasks. Raw and MARGIN selection use the same answer-grouping and voting rule; only the confidence weights differ.

Table~\ref{tab:audited-primary} shows lower ECE on all three benchmarks. Pairwise resolution increases by 35.13, 31.31, and 20.99 percentage points on HumanEval, MBPP, and BigCodeBench, respectively, with positive paired difference intervals. Answer-selection accuracy rises from 87.94\% to 92.26\% on MBPP and from 13.89\% to 27.87\% on BigCodeBench, also with positive paired difference intervals. The HumanEval selection difference is inconclusive. Raw confidence leaves the leading answer tied on 62.2\%, 59.9\%, and 78.5\% of HumanEval, MBPP, and BigCodeBench tasks, so the raw result depends on how ties are resolved. Table~\ref{supp:ties} in Appendix~\ref{suppsec:other} repeats the comparison with ties resolved uniformly in expectation. The selection differences are then 8.79, 10.60, and 12.75 percentage points, each with a positive paired interval. Thus corrected confidence can improve the answer chosen by the coordinator, but improving calibration does not establish an improvement in every downstream statistic. These selection results come from the separate cold-start experiments and are distinct from the post-shift comparison in Section~\ref{sec:online-family}.

\begin{table}[!htbp]
\caption{Separate code-generation coordination experiments, initialised without prior history, on the 18-model pool. ECE is lower-is-better; selection and pairwise resolution are percentages where higher is better. All deltas are MARGIN minus Raw.}\label{tab:audited-primary}
\centering
\small
\setlength{\tabcolsep}{3pt}
\begin{tabular}{@{}llrrrr@{}}
\toprule
Benchmark & Metric & Raw & MARGIN & $\Delta$ & 95\% interval \\
\midrule
HumanEval & ECE & 12.80 & 2.02 & $-10.78$ & $[-11.15,-8.43]$ \\
HumanEval & Selection & 94.51 & 95.33 & $+0.82$ & $[-1.13,+3.04]$ \\
HumanEval & Pairwise & 46.53 & 81.66 & $+35.13$ & $[+31.55,+38.93]$ \\
\addlinespace
MBPP & ECE & 17.51 & 1.90 & $-15.60$ & $[-15.77,-12.20]$ \\
MBPP & Selection & 87.94 & 92.26 & $+4.32$ & $[+1.87,+6.96]$ \\
MBPP & Pairwise & 49.94 & 81.25 & $+31.31$ & $[+28.36,+34.19]$ \\
\addlinespace
BigCodeBench & ECE & 72.72 & 13.04 & $-59.68$ & $[-60.41,-57.00]$ \\
BigCodeBench & Selection & 13.89 & 27.87 & $+13.98$ & $[+7.90,+19.97]$ \\
BigCodeBench & Pairwise & 42.37 & 63.36 & $+20.99$ & $[+15.39,+26.44]$ \\
\bottomrule
\end{tabular}
\end{table}

\FloatBarrier
\subsection{Confidence Inversion in the Heterogeneous Pool}\label{sec:inversion}

BigCodeBench provides the clearest diagnosis of the raw coordination signal. Across the 18-model pool, mean stated confidence is negatively correlated with observed pass rate ($r=-0.692$, 95\% interval $[-0.759,-0.514]$). Raw pairwise resolution is 42.37\% with interval $[38.99,45.58]$: among pairs with one correct answer, choosing the more confident responder is worse than random on this support.

The between-model correlation remains negative in the nine-cloud subset ($r=-0.574$, interval $[-0.775,-0.081]$). Cloud-only pairwise resolution, however, is 45.73\% with interval $[40.07,51.43]$, so below-chance pairwise performance is not established for that subset. These are different statistics: one relates model-level means, while the other compares responses to the same problem. Neither identifies the relative causal contributions of model capability, task characteristics, or serving configuration.

The calibration and selection improvements in Table~\ref{tab:audited-primary} use the same retained support and do not require attributing that raw failure to a particular cause. Appendix~\ref{suppsec:support} gives the exact cohort identifiers.

\subsection{Calibration without Prior History}\label{sec:cold-start}

The single-benchmark online comparisons answer a different question from adaptation of previously learned state. Table~\ref{tab:online-cold} shows that MARGIN has lower ECE than decayed histogram on HumanEval and MBPP, while its comparisons with the other four methods are inconclusive. On BigCodeBench, online Platt, sliding-window histogram, and both windowed-accuracy variants have lower ECE than MARGIN; the decayed-histogram comparison is inconclusive.

This contrasts with the post-shift ordering in Table~\ref{tab:online-warm}. The matched nine-cloud BCB cold-start comparison in Table~\ref{supp:matched-cloud} retains the four losses, so changing from the 18-model pool to the nine-cloud pool alone does not explain the contrast. The result concerns ECE accumulated over the evaluated stream; it does not, by itself, establish a difference in convergence rate or isolate forgetting schedule as the causal explanation.

\begin{table}[!htbp]
\caption{Online comparison within a single benchmark, without prior calibration history, on the 18-model pool. Cells are MARGIN minus comparator ECE with paired 95\% intervals; negative values favour MARGIN.}\label{tab:online-cold}
\centering
\footnotesize
\setlength{\tabcolsep}{3pt}
\begin{tabular}{@{}lrrr@{}}
\toprule
Comparator & HumanEval & MBPP & BigCodeBench \\
\midrule
Decayed histogram & $-4.66$ $[-5.29,-3.00]$ & $-4.03$ $[-4.61,-2.55]$ & $-0.63$ $[-4.79,+3.99]$ \\
Online Platt & $-1.53$ $[-4.32,+2.73]$ & $+0.70$ $[-1.91,+2.71]$ & $+3.79$ $[+2.94,+3.99]$ \\
Sliding-window hist. & $-0.14$ $[-1.72,+0.88]$ & $-0.20$ $[-1.17,+0.87]$ & $+7.98$ $[+3.66,+10.11]$ \\
Windowed acc. multiply & $+0.79$ $[-1.08,+1.55]$ & $+0.57$ $[-0.83,+1.48]$ & $+8.98$ $[+4.52,+11.08]$ \\
Windowed acc. replace & $-0.14$ $[-1.76,+0.98]$ & $-0.10$ $[-1.25,+1.08]$ & $+8.76$ $[+4.23,+10.93]$ \\
\bottomrule
\end{tabular}
\end{table}

\FloatBarrier
\subsection{Feedback Availability and Selection Censoring}\label{sec:feedback}

The primary experiments update all usable responders. The selected-model-only diagnostic keeps full-feedback source warmup, then updates only the responder with the highest calibrated confidence on each target problem. Ties follow the candidate order specified by each protocol. This diagnostic uses a stable model-level argmax in place of the answer-level vote evaluated in Section~\ref{sec:coordination}. ECE still covers all usable candidate predictions.

Table~\ref{tab:selected-feedback} reports the BCB transitions under both retained ordering protocols. MARGIN's ECE increases by approximately 30--34 percentage points relative to its matched full-feedback trajectory. Feedback concentration changes the evidence the calibrator receives: seldom-selected models cannot regularly refresh their estimates on the full workload. Outcome availability is therefore a substantive part of the coordinator's operating conditions.

Absolute error and sensitivity to censoring must be distinguished. MARGIN retains lower selected-regime ECE than the five listed baselines, even though its degradation is larger than that of the hard-window methods. Online Platt degrades more than MARGIN; the difference in degradation relative to decayed histogram is inconclusive. Tables~\ref{supp:feedback}--\ref{supp:vulnerability} in Appendix~\ref{suppsec:feedback} give the paired degradation intervals, all six MARGIN conditions, and the comparator differences. The experiment diagnoses the limitation; it does not evaluate an exploration or off-policy correction.

\begin{table}[!htbp]
\caption{Selected-model-only feedback on the BCB transitions. Each cell is selected-regime ECE with its increase from matched full feedback in parentheses. The selected model is each method's stable calibrated-argmax responder, which can differ from the winner of the answer-level vote. Lower absolute ECE and smaller degradation are distinct comparisons.}\label{tab:selected-feedback}
\centering
\footnotesize
\setlength{\tabcolsep}{3pt}
\begin{tabular}{@{}lrrrr@{}}
\toprule
Method & HE$\to$BCB & MBPP$\to$BCB & HE$\to$BCB & MBPP$\to$BCB \\
\midrule
 & \multicolumn{2}{c}{Problem ordering} & \multicolumn{2}{c}{Historical ordering} \\
\midrule
MARGIN & 46.48 ($+33.49$) & 41.75 ($+29.86$) & 46.62 ($+33.52$) & 41.66 ($+29.74$) \\
Decayed histogram & 60.26 ($+32.19$) & 53.25 ($+28.16$) & 60.19 ($+32.19$) & 52.95 ($+28.01$) \\
Online Platt & 55.46 ($+40.13$) & 48.50 ($+35.08$) & 55.62 ($+40.19$) & 48.52 ($+35.07$) \\
Sliding-window hist. & 62.38 ($+21.31$) & 55.08 ($+17.22$) & 62.40 ($+21.26$) & 55.02 ($+17.13$) \\
Windowed acc. multiply & 60.03 ($+18.59$) & 54.75 ($+16.29$) & 60.07 ($+18.57$) & 54.68 ($+16.23$) \\
Windowed acc. replace & 63.86 ($+20.35$) & 58.29 ($+17.85$) & 63.91 ($+20.34$) & 58.23 ($+17.80$) \\
\bottomrule
\end{tabular}
\end{table}

\FloatBarrier
\subsection{Parameter Sensitivity}\label{sec:ablation}\label{sec:abl-alpha}\label{sec:abl-bands}\label{sec:abl-asymmetric}\label{sec:abl-shrinkage}

All main comparisons keep $K=3$, $\alpha=0.04$, and $k_s=100$ fixed. Table~\ref{tab:alpha-sensitivity} shows that $\alpha=0.005$ produces much higher ECE on the BCB transitions, while $\alpha=0.08$ improves on $0.04$. These sensitivity results use historical-observation ordering rather than the primary problem ordering; their matched $0.04$ values should not be substituted into Table~\ref{tab:audited-shifts}.

Shrinkage is likewise a design choice whose benefit depends on the condition. On the primary problem-order shift grid, removing shrinkage yields inconclusive differences on the BCB, question-answering, and mathematics transitions, and lower ECE on both same-task hardening checks. Those six ablation rows accompany the online comparisons in Table~\ref{supp:full-shifts}. Table~\ref{supp:parameters} in Appendix~\ref{suppsec:parameters} gives the complete retained parameter grid. The findings characterise the fixed configuration; they do not select a new one or establish an optimal default for arbitrary workloads.

Band count matters little while the partition stays coarse. Between one and five bands, ECE on the BCB, question-answering, and mathematics transitions changes by less than 0.3 points, and one band is 0.05 points lower than three on HE$\to$BCB. On the hardening checks, three bands are 0.54 points lower than one on MBPP$\to$MBPP+, and the HE$\to$HE+ difference is inconclusive. Ten and twenty bands raise ECE by 2.9 to 6.9 points on the BCB transitions. The grid therefore supports a small number of bands without singling out three.

Table~\ref{supp:parameters} also compares symmetric and asymmetric success/failure update rates on HE$\to$BCB and MBPP$\to$MBPP+. Several asymmetric settings substantially increase ECE relative to the fixed symmetric setting. These experiments complement Theorem~\ref{prop:symmetric}'s stationary bias analysis by evaluating outcome-dependent updates on the recorded workloads.

\begin{table}[!htbp]
\caption{Learning-rate sensitivity on the BCB transitions under historical-observation ordering. Delta is ECE minus the fixed $\alpha=0.04$ configuration. These baselines differ slightly from Table~\ref{tab:audited-shifts}, which uses problem ordering.}\label{tab:alpha-sensitivity}
\centering
\small
\setlength{\tabcolsep}{3pt}
\begin{tabular}{@{}lrrrr@{}}
\toprule
Condition & $\alpha$ & ECE & $\Delta$ & 95\% interval \\
\midrule
HE$\to$BCB & 0.005 & 47.05 & $+33.96$ & $[+30.21,+34.58]$ \\
HE$\to$BCB & 0.04 & 13.10 & $+0.00$ & --- \\
HE$\to$BCB & 0.08 & 10.36 & $-2.74$ & $[-4.67,-1.02]$ \\
MBPP$\to$BCB & 0.005 & 43.75 & $+31.84$ & $[+27.96,+32.50]$ \\
MBPP$\to$BCB & 0.04 & 11.91 & $+0.00$ & --- \\
MBPP$\to$BCB & 0.08 & 9.80 & $-2.11$ & $[-4.07,-0.45]$ \\
\bottomrule
\end{tabular}
\end{table}

\FloatBarrier
\section{Discussion and Limitations}\label{sec:discussion}

For a coordinator that uses confidence to compare answers, calibration history is useful only insofar as it remains relevant to the workload. The shift experiments evaluate precisely that situation: each method arrives at the new workload with estimates learned on earlier tasks. MARGIN's lower post-shift ECE against the tested online baselines in the coding transitions, and against four in the MMLU-to-MMLU-Pro transition, supports its use as an adaptive correction in these settings. The different cold-start ordering shows why calibration learned from an empty history should not be treated as interchangeable with adapting an existing correction.

The MAS contribution lies in the coordinator's interpretation of heterogeneous signals. MARGIN assigns a model-specific meaning to a stated confidence and updates that interpretation from outcomes. The BCB inversion demonstrates a directional failure of raw confidence, and the separate answer-selection experiments show that correction can affect the collective decision. Those gains are relative to uncalibrated confidence weighting. They do not establish that an ensemble is preferable to deploying a known strongest model at lower inference cost. This evidence supports the evaluated coordinator/responder interface; it is not evidence about negotiation, interactive deliberation, or autonomous multi-step agents.

The comparisons assess complete fixed configurations. They do not isolate a universal benefit of band stratification, shrinkage, or a particular forgetting schedule. Theorem~\ref{prop:symmetric} explains an accuracy-estimation bias under stationarity. The comparative adaptation result and the quality of the final vote rest on the experiments. The cold-start comparisons and parameter ablations show where the fixed MARGIN configuration is not preferable.

Correctness feedback is the main operational condition. The full-feedback setting is applicable when candidate responses can be checked, as in the retained executable-test experiments. A deployment that verifies only its chosen responder generates a selectively observed history. The diagnostic shows substantial calibration damage in that setting. Exploration, periodic verification of alternative responses, or off-policy adjustment are possible follow-on approaches, but are not evaluated here.

Generalisation is limited by the archived cohort, retained task support, and abrupt workload transitions. Test hardening reuses tasks; gradual or recurring drift is untested. Provider, model capability, quantisation, and serving configuration were not varied independently, and per-request context and termination telemetry are incomplete. The study therefore cannot establish that truncation did not occur or identify the serving causes of inversion. The bootstrap intervals quantify variation conditional on the recorded histories and support, with no simultaneous or new-model-population guarantee.

\section{Conclusion}\label{sec:conclusion}

For coordinators that use self-reported confidence to compare model answers, both model heterogeneity and workload changes can undermine a common confidence scale. The BigCodeBench confidence inversion demonstrates the resulting risk to coordination. MARGIN learns model-specific, confidence-dependent corrections from observed outcomes while leaving the responding models unchanged. Under matched feedback and retained source histories, it has lower post-shift calibration error than all five online baselines in two coding transitions and than four in a question-answering transition. Separate code-generation experiments show pairwise improvements on all three benchmarks and answer-selection gains on two. Together, these results establish an empirically evaluated runtime calibration mechanism for heterogeneous coordination under changing task distributions, with its applicability conditioned on available correctness feedback.

\section*{Declarations}\label{sec:declarations}

\noindent\textbf{Funding.} No external funding was received.

\noindent\textbf{Data and code availability.} The code and retained analysis data supporting the reported results are not publicly available.


\setcounter{maintablecount}{\value{table}}
\begin{appendices}
\setcounter{table}{\value{maintablecount}}
\renewcommand{\thetable}{\arabic{table}}

\section{Theoretical Material}\label{app:proofs}

\subsection*{A.1 Proof of Theorem~\ref{prop:symmetric} (Unbiasedness of the Symmetric Update)}

For one band, write the asymmetric update as
\[
\hat a_t=
\begin{cases}
(1-\alpha_{\mathrm{up}})\hat a_{t-1}+\alpha_{\mathrm{up}}, & X_t=1,\\
(1-\alpha_{\mathrm{down}})\hat a_{t-1}, & X_t=0,
\end{cases}
\]
with $X_t\sim\mathrm{Bernoulli}(\theta)$ i.i.d.  Because $X_t$ is independent of the previous state, the stationary mean $\mu=\mathbb E[\hat a_\infty]$ satisfies
\[
\mu=\theta\bigl[(1-\alpha_{\mathrm{up}})\mu+\alpha_{\mathrm{up}}\bigr]
 +(1-\theta)(1-\alpha_{\mathrm{down}})\mu.
\]
Rearranging gives
\[
\mu=\frac{\alpha_{\mathrm{up}}\theta}
{\alpha_{\mathrm{up}}\theta+\alpha_{\mathrm{down}}(1-\theta)},
\]
which is Eq.~\eqref{eq:asymmetric-ss}.  Setting $\mu=\theta$ for $0<\theta<1$ requires
\[
\alpha_{\mathrm{up}}\theta(1-\theta)
 =\alpha_{\mathrm{down}}\theta(1-\theta),
\]
and hence $\alpha_{\mathrm{up}}=\alpha_{\mathrm{down}}$.  Subtracting $\theta$ from the stationary mean and taking absolute values yields Eq.~\eqref{eq:asymmetric-bias}.  The result is confined to unbiasedness within this two-rate family under stationary i.i.d. Bernoulli outcomes; it does not assert MSE optimality under drift or optimality of the downstream ratio or vote. \qed

\section{Experimental Protocol}\label{app:protocol}

\subsection{Models, Workloads, and Retained Support}\label{suppsec:support}

Table~\ref{supp:models} lists all 18 model identifiers and the available observation counts. Cloud service is identified operationally, because an API identifier does not establish all serving parameters, including precision. The nine local variants were recorded as 4-bit models, with eight Q4\_K\_M tags and DeepSeek Coder V2 at Q4\_0. Local models were served with Ollama across two binary-version families, 0.13.x and 0.16.x, whose server-default context sizes differ, 2048 and 4096 respectively. Complete per-request context, token, and termination telemetry is unavailable. The cloud/local split is therefore descriptive. It was never a controlled intervention on quantisation or model capability.

\begingroup
\footnotesize
\setlength{\tabcolsep}{4pt}
\begin{longtable}{@{}p{.58\linewidth}rrr@{}}
\caption{Retained candidate counts by model and benchmark: determinate correctness / available confidence-and-correctness pairs. Counts are before aggregation across histories. The nine cloud identifiers are the rows without the \texttt{ollama/} prefix. Full identifiers define the canonical pool.}\label{supp:models}\\
\toprule
Model identifier & HumanEval & MBPP & BCB \\
\midrule
\endfirsthead
\multicolumn{4}{@{}l}{\textit{Table~\thetable\ continued}} \\
\toprule
Model identifier & HumanEval & MBPP & BCB \\
\midrule
\endhead
\path{MiniMaxAI/MiniMax-M2.5} & 164 / 158 & 257 / 254 & 141 / 93 \\
\path{Qwen/Qwen2.5-14B-Instruct-1M} & 164 / 164 & 257 / 257 & 141 / 141 \\
\path{Qwen/Qwen2.5-32B-Instruct} & 164 / 164 & 257 / 257 & 141 / 141 \\
\path{Qwen/Qwen3-32B} & 164 / 163 & 257 / 257 & 141 / 140 \\
\path{Qwen/Qwen3-Coder-480B-A35B-Instruct-FP8} & 164 / 163 & 257 / 257 & 142 / 142 \\
\path{deepseek-ai/DeepSeek-R1-Distill-Qwen-32B} & 164 / 55 & 257 / 102 & 142 / 126 \\
\path{deepseek-ai/DeepSeek-V3.2} & 164 / 164 & 257 / 257 & 141 / 141 \\
\path{ollama/deepseek-coder-v2} & 164 / 164 & 257 / 257 & 141 / 141 \\
\path{ollama/llama3.1_70b} & 164 / 163 & 257 / 256 & 142 / 136 \\
\path{ollama/llama3.1_8b} & 164 / 164 & 257 / 257 & 142 / 142 \\
\path{ollama/mistral_7b} & 164 / 164 & 257 / 253 & 141 / 141 \\
\path{ollama/phi4_14b} & 164 / 164 & 257 / 221 & 141 / 140 \\
\path{ollama/qwen2.5-coder_32b} & 164 / 164 & 257 / 257 & 141 / 141 \\
\path{ollama/qwen2.5-coder_7b} & 164 / 164 & 257 / 257 & 143 / 143 \\
\path{ollama/qwen2.5_14b} & 164 / 157 & 257 / 257 & 142 / 136 \\
\path{ollama/qwen2.5_72b} & 164 / 163 & 257 / 255 & 142 / 139 \\
\path{openai/gpt-oss-120b} & 164 / 132 & 257 / 246 & 142 / 133 \\
\path{zai-org/GLM-4.7-Flash} & 164 / 157 & 257 / 255 & 140 / 133 \\
\bottomrule
\end{longtable}
\endgroup
\begin{table}[!htbp]
\caption{Task and confidence/outcome support for the displayed analyses. A task count is not a claim that every model has a usable record on every task.}\label{supp:support}
\centering
\small
\setlength{\tabcolsep}{3pt}
\begin{tabular}{@{}llrr@{}}
\toprule
Analysis & Cohort & Target tasks & Usable target records \\
\midrule
HumanEval (cold start) & 18 models & 164 & 2787 \\
MBPP (cold start) & 18 models & 257 & 4412 \\
BigCodeBench (cold start) & 18 models & 144 & 2449 \\
HE$\to$BCB & 9 cloud & 144 & 1190 \\
MBPP$\to$BCB & 9 cloud & 144 & 1190 \\
MMLU-shift & 9 cloud & 200 & 1562 \\
GSM8K$\to$MATH & 9 cloud & 196 & 1700 \\
HE$\to$HE+ & 9 cloud & 163 & 1467 \\
MBPP$\to$MBPP+ & 9 cloud & 224 & 2016 \\
\bottomrule
\end{tabular}
\end{table}

\FloatBarrier

HumanEval (HE) and Mostly Basic Python Problems (MBPP) are the code-generation source benchmarks. BigCodeBench (BCB) supplies the hard coding target; the final evaluation used the frozen BigCodeBench-Hard-v0.1.4 test/reference snapshot. Four tasks are excluded because their references did not satisfy the frozen offline/resource contract: BigCodeBench/101, /1012, /417, and /590. Other unresolved candidate outcomes are excluded at the record level. Passing a retained test suite establishes conformance with those tests, which falls short of proof for all possible program behaviour.

HumanEval+ excludes HumanEval/32 and retains 163 tasks. MBPP+ retains 224 tasks with matched expanded tests. These tests reuse the original candidate programs rather than eliciting new answers. The archived MMLU-shift workload moves from MMLU STEM questions to MMLU-Pro business questions, jointly changing benchmark, subject, and answer format. MMLU-Pro introduces harder questions and up to ten answer options. The mathematics transition changes from the grade-school problems of GSM8K to the competition problems of MATH. The question-answering target has 200 unique task identifiers with the \texttt{mmlupro\_} prefix; all 1,795 target records across the nine models, before confidence-availability filtering, carry the archived subject label \texttt{business}. The QA/mathematics labels are reproducible from the saved answers, gold records, and deterministic scoring rules, but that validation does not authenticate the complete original request or serving environment.

LiveCodeBench and CodeContests lack sufficiently authenticated historical test/evaluation state. TriviaQA lacks the accepted-alias evidence needed for part of its historical labels. Their original dynamic-pool combinations inherit those limitations. Those conditions are not included in the primary or shift comparisons.

\subsection{Elicitation and Decoding}\label{app:prompts}
Verbalized collection requested an answer followed by confidence on a 0--100 scale, using temperature zero. Ordinary coding output caps were 4096 tokens, with a 16384-token cap for designated high-token paths; multiple choice used 2048 and mathematics 16384. A configured output cap does not establish how many tokens were actually generated or that no truncation occurred. The following are the archived templates for the retained workload families. Braces mark benchmark-specific fields and add no prompt content.

\Needspace{12\baselineskip}
\paragraph{HumanEval.}
\begin{Verbatim}[fontsize=\footnotesize,breaklines=true]
You are a Python programming assistant. Write a Python function that solves the following problem. Then state your confidence that your solution is correct as a number from 0 to 100.

{problem prompt}

Respond in exactly this format:
```python
{your solution}
```
Confidence: {0-100}
\end{Verbatim}
\Needspace{12\baselineskip}
\paragraph{MBPP.}
\begin{Verbatim}[fontsize=\footnotesize,breaklines=true]
You are a Python programming assistant. Write a Python function that solves the following problem. Then state your confidence that your solution is correct as a number from 0 to 100.

{problem prompt}

Your function should pass these test cases:
{first test case}

Respond in exactly this format:
```python
{your solution}
```
Confidence: {0-100}
\end{Verbatim}
\Needspace{12\baselineskip}
\paragraph{BigCodeBench.}
\begin{Verbatim}[fontsize=\footnotesize,breaklines=true]
You are a Python programming assistant. Write a Python function called `task_func` that solves the following problem. Then state your confidence that your solution is correct as a number from 0 to 100.

Problem description:
{instruct prompt}

Required libraries: {libraries}

The function MUST be named `task_func`.

Respond in exactly this format:
```python
{your complete solution including imports}
```
Confidence: {0-100}
\end{Verbatim}
\Needspace{12\baselineskip}
\paragraph{MMLU and MMLU-Pro.}
\begin{Verbatim}[fontsize=\footnotesize,breaklines=true]
{question text}
First, state your answer as a single letter (A, B, C, D, etc.).
Then, state your confidence in that answer as a number from 0 to 100.

Format your response exactly like this:
Answer: X
Confidence: Y
\end{Verbatim}
\Needspace{12\baselineskip}
\paragraph{Mathematics.}
\begin{Verbatim}[fontsize=\footnotesize,breaklines=true]
Solve the following math problem step by step. The last line of your response should be of the form "Answer: $ANSWER" where $ANSWER is your final answer. Do not use a \boxed command. After your answer line, add a confidence line.

Problem: {question text}

The last two lines of your response must be exactly:
Answer: <your final answer>
Confidence: <0-100>
\end{Verbatim}

\subsection{Correctness Labels and Vote Grouping}\label{app:scoring}

Analysis consumes the archived extracted confidence score. Correctness requires a determinate retained label, and calibration additionally requires valid confidence. Table~\ref{supp:models} therefore distinguishes the outcome and calibration denominators. Unknown, quarantined, and unresolved labels are excluded rather than converted to failure. Original request counts and complete serving telemetry are not recoverable from the retained observations alone.

For code generation, benchmark evaluators determine correctness. Exact-code signatures serve only to group votes. They normalise line endings and strip leading and trailing newline characters while preserving internal whitespace. A valid alternative program can pass the evaluator without matching another candidate's signature. Empty programs form no vote group, and a task with no usable group contributes zero selection success.

\subsection{Comparator Definitions}\label{suppsec:baselines}

All methods operate per model on confidence/outcome pairs. Formulas in this section use confidence $c\in[0,1]$; the implementation's equivalent external scale is 0--100. Predictions are made before the corresponding outcomes are incorporated. The online methods carry their own state from source to target without reset. The settings below are fixed configurations, chosen without a target-specific hyperparameter search.

\paragraph{Sliding-window histogram.}
Partition confidence into ten equal-width bins, with the last bin including one. For each model, retain its last $W=200$ observations. A new confidence in bin $b$ is replaced by $h_b/n_b$, the success count divided by the count of prior observations in that bin within the model's current window. An empty bin returns raw confidence. Every new model observation enters the window, and the oldest is evicted when its size exceeds 200.

\paragraph{Decayed histogram.}
Use the same ten fixed bins and a separate success estimate per model and bin. Before the first observation in a bin, return raw confidence; the first update assigns its observed outcome directly to the bin estimate. Subsequent updates use $p_b\leftarrow0.98p_b+0.02o$. Prediction uses $p_b$. Thus its decay, binning, and initialisation differ from MARGIN; it is not a controlled ablation of the ratio alone.

\paragraph{Windowed accuracy, replace.}
For each model, keep the last 200 confidence/outcome pairs. Replace the new score by the window's mean correctness $\bar o_W$. With an empty window, return raw confidence. The estimate does not condition on a confidence bin.

\paragraph{Windowed accuracy, multiply.}
Use the same window but return $\operatorname{clip}(c\bar o_W/\bar c_W,0,1)$, where $\bar c_W$ is its mean stated confidence. An empty window or nonpositive confidence mean uses the raw score. This shares MARGIN's multiplicative form but uses a single model-wide hard window rather than confidence-band EWMAs.

\paragraph{Online Platt scaling.}
Each model has $p=\sigma(ac+b)$, initialised with $a=1$, $b=0$. After observing $o$, update
\[
 a\leftarrow a-\eta[(p-o)c+\lambda a],\qquad
 b\leftarrow b-\eta[(p-o)+\lambda b],
\]
with $\eta=0.1$ and $\lambda=10^{-4}$. This is a single gradient step on regularised log loss per observation, and the full history is never refitted. Its initial prediction is $\sigma(c)$ rather than identity.

\paragraph{MARGIN ablations.}
The primary ratio implementation uses $K=3$, $\alpha=0.04$, and $k_s=100$. State is initialised at confidence-band midpoints and predictions are clipped to $[0,1]$. Both denominators are safeguarded below by $10^{-12}$. The model-level shrinkage target is the ratio of count-weighted band totals, as in the manuscript. With $k_s=0$, the effective correction is the band ratio; all other settings remain unchanged. Asymmetric-rate variants use the specified rate after success or failure for both EWMA states, while retaining the remaining construction.

\subsection{Order, Aggregation, and Intervals}\label{suppsec:aggregation}

The primary problem protocol permutes target problem identifiers with seeds 1--100 and source identifiers with seeds 10001--10100. Every method uses the same permutations. All usable scores for a target problem are calculated before any outcome of that problem updates a calibrator. Model ordering within the problem is the fixed sorted identifier order.

Historical-observation ordering instead permutes the archived source and target observation lists. Code-generation source/target seeds are 1--100/101--200; QA/mathematics seeds are 42--141/142--241. Under full feedback, a prediction is followed by its outcome update before the next observation is processed. A problem's different model observations can therefore occur at different positions in that history. Under selected feedback, the permuted target list is grouped by first problem appearance, all available candidate scores for that problem are computed, and only the stable argmax is updated. Equal-score ties follow the candidate order induced by that recorded list. These historical conventions serve as sensitivity analyses, and the primary all-predict-before-any-update protocol remains the reference. Parameter and delay grids use historical-observation ordering.

Let $C_{t,b}^{(s)}$, $P_{t,b}^{(s)}$, and $Y_{t,b}^{(s)}$ denote bin counts, prediction sums, and correctness sums for problem $t$, bin $b$, and history $s$. Write bars for their averages over the 100 histories. The reported ECE is
\[
 100\,\frac{\sum_b|\sum_t\bar Y_{t,b}-\sum_t\bar P_{t,b}|}
 {\sum_{t,b}\bar C_{t,b}}.
\]
The absolute value is applied after pooling the histories. Averaging 100 ECE values first gives a different estimand and is not used in the primary tables.

For selection, the saved binary success contribution is averaged over histories and then over retained problems. For pairwise resolution, the mean-over-history win credits are summed and divided by the corresponding sum of eligible-pair counts. An eligible pair contains exactly one correct response and available confidence on both sides; equal confidence contributes half a win. The statistic is therefore a ratio of pooled contributions, which differs from an unweighted average of per-problem percentages. Code-answer equality is only a vote-grouping relation. Empty programs form no group, and an abstained task contributes zero selection success.

Bootstrap intervals use 10,000 problem-cluster draws with NumPy RandomState seed 42. Each draw resamples retained problem indices with replacement and keeps all methods' averaged contributions paired at those indices. ECE and pairwise ratios are recomputed for the resampled contributions. Calibration trajectories are not replayed inside each draw. These pointwise intervals describe uncertainty conditional on the retained observations and recorded histories; they do not test equivalence or provide simultaneous coverage over every comparison.

\FloatBarrier

\section{Extended Comparisons and Sensitivity Analyses}\label{app:extended}

\subsection{Complete Distribution-Shift Comparisons}\label{suppsec:shifts}

Table~\ref{supp:full-shifts} reports all six retained conditions, rather than only the coding and question-answering rows highlighted in the main comparative table. MARGIN has lower ECE than all five online baselines after both coding transitions and than four on the question-answering transition. The mathematics/online-Platt comparison is also favourable by its pointwise interval; other mathematics comparisons and the five-baseline hardening comparisons are inconclusive. The no-shrinkage rows have a different status: they are a method ablation, and favour removing shrinkage on the two hardening conditions.

\begingroup
\small
\setlength{\tabcolsep}{4pt}
\begin{longtable}{@{}lrrrr@{}}
\caption{Complete post-shift ECE comparisons on the nine-cloud pool under problem ordering. Negative MARGIN-minus-comparator deltas favour MARGIN. The no-shrinkage ratio is an ablation, separate from the five online baselines. Three-decimal intervals preserve the sign of the mathematics/online-Platt endpoint close to zero. All intervals are pointwise, without multiplicity adjustment.}\label{supp:full-shifts}\\
\toprule
Comparator & MARGIN & Comparator ECE & $\Delta$ & 95\% interval \\
\midrule
\endfirsthead
\multicolumn{5}{@{}l}{\textit{Table~\thetable\ continued}} \\
\toprule
Comparator & MARGIN & Comparator ECE & $\Delta$ & 95\% interval \\
\midrule
\endhead
\multicolumn{5}{@{}l}{\textbf{HE$\to$BCB}} \\*
Raw & 12.99 & 66.79 & $-53.802$ & $[-54.408,-50.165]$ \\
Decayed histogram & 12.99 & 28.07 & $-15.075$ & $[-15.983,-11.841]$ \\
Online Platt & 12.99 & 15.32 & $-2.333$ & $[-2.630,-1.934]$ \\
Sliding-window hist. & 12.99 & 41.06 & $-28.072$ & $[-29.119,-24.692]$ \\
Windowed acc. multiply & 12.99 & 41.44 & $-28.447$ & $[-29.030,-24.767]$ \\
Windowed acc. replace & 12.99 & 43.51 & $-30.517$ & $[-31.211,-26.742]$ \\
MARGIN, $k_s=0$ & 12.99 & 13.05 & $-0.058$ & $[-0.176,+0.006]$ \\
\addlinespace
\multicolumn{5}{@{}l}{\textbf{MBPP$\to$BCB}} \\*
Raw & 11.89 & 66.79 & $-54.900$ & $[-55.554,-51.126]$ \\
Decayed histogram & 11.89 & 25.09 & $-13.196$ & $[-14.183,-10.003]$ \\
Online Platt & 11.89 & 13.42 & $-1.528$ & $[-1.867,-1.072]$ \\
Sliding-window hist. & 11.89 & 37.86 & $-25.967$ & $[-27.135,-22.519]$ \\
Windowed acc. multiply & 11.89 & 38.46 & $-26.569$ & $[-27.207,-22.750]$ \\
Windowed acc. replace & 11.89 & 40.44 & $-28.548$ & $[-29.283,-24.649]$ \\
MARGIN, $k_s=0$ & 11.89 & 11.94 & $-0.048$ & $[-0.166,+0.013]$ \\
\addlinespace
\multicolumn{5}{@{}l}{\textbf{MMLU-shift}} \\*
Raw & 2.85 & 18.70 & $-15.846$ & $[-17.004,-12.502]$ \\
Decayed histogram & 2.85 & 7.18 & $-4.329$ & $[-5.891,-1.452]$ \\
Online Platt & 2.85 & 3.17 & $-0.319$ & $[-1.195,+0.418]$ \\
Sliding-window hist. & 2.85 & 7.20 & $-4.343$ & $[-5.514,-1.564]$ \\
Windowed acc. multiply & 2.85 & 7.37 & $-4.514$ & $[-5.170,-2.586]$ \\
Windowed acc. replace & 2.85 & 6.87 & $-4.013$ & $[-5.384,-0.966]$ \\
MARGIN, $k_s=0$ & 2.85 & 2.88 & $-0.030$ & $[-0.190,+0.098]$ \\
\addlinespace
\multicolumn{5}{@{}l}{\textbf{GSM8K$\to$MATH}} \\*
Raw & 3.45 & 26.16 & $-22.703$ & $[-24.164,-17.172]$ \\
Decayed histogram & 3.45 & 5.12 & $-1.667$ & $[-3.126,+0.594]$ \\
Online Platt & 3.45 & 4.63 & $-1.177$ & $[-1.610,-0.003]$ \\
Sliding-window hist. & 3.45 & 8.70 & $-5.244$ & $[-6.852,+0.044]$ \\
Windowed acc. multiply & 3.45 & 8.08 & $-4.629$ & $[-6.099,+0.869]$ \\
Windowed acc. replace & 3.45 & 8.28 & $-4.828$ & $[-6.439,+0.426]$ \\
MARGIN, $k_s=0$ & 3.45 & 3.40 & $+0.057$ & $[-0.027,+0.171]$ \\
\addlinespace
\multicolumn{5}{@{}l}{\textbf{HE$\to$HE+}} \\*
Raw & 3.57 & 8.57 & $-4.999$ & $[-6.248,-1.987]$ \\
Decayed histogram & 3.57 & 2.81 & $+0.764$ & $[-0.810,+2.525]$ \\
Online Platt & 3.57 & 1.94 & $+1.634$ & $[-0.252,+3.343]$ \\
Sliding-window hist. & 3.57 & 2.71 & $+0.857$ & $[-0.870,+3.176]$ \\
Windowed acc. multiply & 3.57 & 4.75 & $-1.182$ & $[-2.114,+0.763]$ \\
Windowed acc. replace & 3.57 & 4.17 & $-0.600$ & $[-2.557,+2.527]$ \\
MARGIN, $k_s=0$ & 3.57 & 2.25 & $+1.318$ & $[+0.725,+1.702]$ \\
\addlinespace
\multicolumn{5}{@{}l}{\textbf{MBPP$\to$MBPP+}} \\*
Raw & 3.07 & 17.96 & $-14.894$ & $[-16.179,-9.820]$ \\
Decayed histogram & 3.07 & 4.71 & $-1.642$ & $[-2.515,+1.112]$ \\
Online Platt & 3.07 & 2.79 & $+0.279$ & $[-1.155,+2.027]$ \\
Sliding-window hist. & 3.07 & 4.02 & $-0.946$ & $[-2.857,+2.497]$ \\
Windowed acc. multiply & 3.07 & 5.06 & $-1.990$ & $[-3.227,+1.317]$ \\
Windowed acc. replace & 3.07 & 4.37 & $-1.298$ & $[-3.312,+2.622]$ \\
MARGIN, $k_s=0$ & 3.07 & 2.19 & $+0.876$ & $[+0.148,+0.954]$ \\
\addlinespace
\bottomrule
\end{longtable}
\endgroup

\FloatBarrier

\subsection{Selected-Model Feedback}\label{suppsec:feedback}

The diagnostic starts with full-feedback source warmup. On the target workload, each method observes all candidates' scores but updates only its own calibrated-argmax responder. Correctness for all usable predictions is retained for evaluating ECE, although most labels are withheld from updates. This is separate from the answer-level voting experiment.

\begingroup
\footnotesize
\setlength{\tabcolsep}{4pt}
\begin{longtable}{@{}llrrrr@{}}
\caption{MARGIN selected-model-only feedback on all six transitions, retaining both ordering protocols. Degradation is selected minus full-feedback ECE.}\label{supp:feedback}\\
\toprule
Condition & Order & Full & Selected & Degradation & 95\% interval \\
\midrule
\endfirsthead
\multicolumn{6}{@{}l}{\textit{Table~\thetable\ continued}} \\
\toprule
Condition & Order & Full & Selected & Degradation & 95\% interval \\
\midrule
\endhead
HE$\to$BCB & Problem & 12.99 & 46.48 & $+33.492$ & $[+29.801,+34.078]$ \\
HE$\to$BCB & Historical & 13.10 & 46.62 & $+33.523$ & $[+29.799,+34.168]$ \\
MBPP$\to$BCB & Problem & 11.89 & 41.75 & $+29.859$ & $[+26.041,+30.499]$ \\
MBPP$\to$BCB & Historical & 11.91 & 41.66 & $+29.743$ & $[+25.888,+30.427]$ \\
MMLU-shift & Problem & 2.85 & 9.05 & $+6.199$ & $[+3.669,+7.317]$ \\
MMLU-shift & Historical & 2.83 & 9.13 & $+6.303$ & $[+4.026,+7.141]$ \\
GSM8K$\to$MATH & Problem & 3.45 & 9.16 & $+5.707$ & $[+0.216,+7.180]$ \\
GSM8K$\to$MATH & Historical & 3.41 & 9.35 & $+5.947$ & $[+0.409,+7.276]$ \\
HE$\to$HE+ & Problem & 3.57 & 3.39 & $-0.184$ & $[-2.287,+1.075]$ \\
HE$\to$HE+ & Historical & 3.59 & 3.43 & $-0.155$ & $[-2.285,+1.091]$ \\
MBPP$\to$MBPP+ & Problem & 3.07 & 4.91 & $+1.844$ & $[-1.025,+3.202]$ \\
MBPP$\to$MBPP+ & Historical & 2.91 & 4.54 & $+1.628$ & $[-0.966,+2.957]$ \\
\bottomrule
\end{longtable}
\endgroup

\begingroup
\small
\setlength{\tabcolsep}{4pt}
\begin{longtable}{@{}lrr@{}}
\caption{Difference in degradation: MARGIN's selected-minus-full change minus the comparator's selected-minus-full change. Positive intervals indicate greater MARGIN degradation; these are not comparisons of absolute selected-regime ECE.}\label{supp:vulnerability}\\
\toprule
Comparator & Difference in degradation & 95\% interval \\
\midrule
\endfirsthead
\multicolumn{3}{@{}l}{\textit{Table~\thetable\ continued}} \\
\toprule
Comparator & Difference in degradation & 95\% interval \\
\midrule
\endhead
\multicolumn{3}{@{}l}{\textbf{HE$\to$BCB; problem ordering}} \\*
Decayed histogram & $+1.303$ & $[-2.000,+2.181]$ \\
Online Platt & $-6.641$ & $[-7.040,-6.326]$ \\
Sliding-window hist. & $+12.180$ & $[+8.780,+13.153]$ \\
Windowed acc. multiply & $+14.904$ & $[+11.181,+15.517]$ \\
Windowed acc. replace & $+13.144$ & $[+9.435,+13.739]$ \\
\multicolumn{3}{@{}l}{\textbf{HE$\to$BCB; historical ordering}} \\*
Decayed histogram & $+1.338$ & $[-1.936,+2.251]$ \\
Online Platt & $-6.664$ & $[-7.064,-6.352]$ \\
Sliding-window hist. & $+12.267$ & $[+8.818,+13.259]$ \\
Windowed acc. multiply & $+14.955$ & $[+11.197,+15.615]$ \\
Windowed acc. replace & $+13.185$ & $[+9.456,+13.833]$ \\
\multicolumn{3}{@{}l}{\textbf{MBPP$\to$BCB; problem ordering}} \\*
Decayed histogram & $+1.697$ & $[-1.820,+2.687]$ \\
Online Platt & $-5.216$ & $[-5.667,-4.853]$ \\
Sliding-window hist. & $+12.637$ & $[+8.996,+13.483]$ \\
Windowed acc. multiply & $+13.570$ & $[+9.723,+14.234]$ \\
Windowed acc. replace & $+12.013$ & $[+8.186,+12.665]$ \\
\multicolumn{3}{@{}l}{\textbf{MBPP$\to$BCB; historical ordering}} \\*
Decayed histogram & $+1.732$ & $[-1.760,+2.752]$ \\
Online Platt & $-5.330$ & $[-5.769,-4.963]$ \\
Sliding-window hist. & $+12.609$ & $[+8.919,+13.434]$ \\
Windowed acc. multiply & $+13.511$ & $[+9.615,+14.208]$ \\
Windowed acc. replace & $+11.939$ & $[+8.076,+12.625]$ \\
\bottomrule
\end{longtable}
\endgroup

\FloatBarrier

\subsection{Parameter Sensitivity}\label{suppsec:parameters}

Table~\ref{supp:parameters} contains every row of the retained parameter matrix. Each comparison uses its own unchanged primary setting under historical-observation ordering. Baseline differences from the problem-order shift table therefore reflect protocol alone, and the frozen values stand. The grid shows sensitivity rather than a universal configuration recommendation. In particular, the $0.08$ rate improves the BCB shift values but worsens the same-task hardening values. Removing shrinkage also has condition-dependent effects. The grid does not select a new primary setting.

\begingroup
\small
\setlength{\tabcolsep}{4pt}
\begin{longtable}{@{}lrrr@{}}
\caption{Complete retained parameter grid, under historical-observation ordering. Deltas compare each setting with the unchanged primary configuration within that condition. The displayed grid is not an optimisation procedure; no parameter is selected from these results. Tiny deltas rounded to zero are not evidence of exact equality.}\label{supp:parameters}\\
\toprule
Setting & ECE & $\Delta$ from primary & 95\% interval \\
\midrule
\endfirsthead
\multicolumn{4}{@{}l}{\textit{Table~\thetable\ continued}} \\
\toprule
Setting & ECE & $\Delta$ from primary & 95\% interval \\
\midrule
\endhead
\multicolumn{4}{@{}l}{\textbf{HE$\to$BCB; learning rate}} \\*
$\alpha=0.005$ & 47.055 & $+33.955$ & $[+30.215,+34.581]$ \\
$\alpha=0.01$ & 35.603 & $+22.503$ & $[+18.756,+23.140]$ \\
$\alpha=0.02$ & 22.691 & $+9.591$ & $[+6.311,+10.251]$ \\
$\alpha=0.04$ (primary) & 13.100 & $+0.000$ & --- \\
$\alpha=0.08$ & 10.362 & $-2.738$ & $[-4.666,-1.018]$ \\
$\alpha=0.16$ & 11.595 & $-1.504$ & $[-4.662,+0.891]$ \\
$\alpha=0.32$ & 16.309 & $+3.210$ & $[-1.004,+5.803]$ \\
\addlinespace
\multicolumn{4}{@{}l}{\textbf{HE$\to$BCB; band count}} \\*
$K=1$ & 13.049 & $-0.050$ & $[-0.072,-0.015]$ \\
$K=2$ & 13.071 & $-0.028$ & $[-0.050,+0.001]$ \\
$K=3$ (primary) & 13.100 & $+0.000$ & --- \\
$K=5$ & 13.275 & $+0.175$ & $[-0.045,+0.241]$ \\
$K=10$ & 16.237 & $+3.138$ & $[+1.862,+3.404]$ \\
$K=20$ & 19.972 & $+6.872$ & $[+4.840,+7.219]$ \\
\addlinespace
\multicolumn{4}{@{}l}{\textbf{HE$\to$BCB; shrinkage}} \\*
$k_s=0$ & 13.150 & $+0.051$ & $[-0.010,+0.163]$ \\
$k_s=5$ & 13.099 & $-0.000$ & $[-0.008,+0.050]$ \\
$k_s=10$ & 13.097 & $-0.003$ & $[-0.009,+0.026]$ \\
$k_s=20$ & 13.096 & $-0.003$ & $[-0.009,+0.005]$ \\
$k_s=50$ & 13.098 & $-0.001$ & $[-0.006,+0.004]$ \\
$k_s=100$ (primary) & 13.100 & $+0.000$ & --- \\
$k_s=200$ & 13.099 & $-0.001$ & $[-0.008,+0.005]$ \\
$k_s=500$ & 13.095 & $-0.004$ & $[-0.018,+0.008]$ \\
$k_s=1000$ & 13.099 & $-0.001$ & $[-0.015,+0.012]$ \\
\addlinespace
\multicolumn{4}{@{}l}{\textbf{HE$\to$BCB; success/failure rates}} \\*
$0.04/0.04$ (primary) & 13.100 & $+0.000$ & --- \\
$0.02/0.08$ & 16.833 & $+3.733$ & $[-4.101,+9.517]$ \\
$0.08/0.02$ & 39.726 & $+26.627$ & $[+22.879,+27.270]$ \\
$0.01/0.04$ & 15.358 & $+2.258$ & $[-3.729,+6.412]$ \\
$0.04/0.01$ & 45.867 & $+32.767$ & $[+29.023,+33.399]$ \\
$0.08/0.04$ & 23.083 & $+9.983$ & $[+6.569,+10.626]$ \\
\addlinespace
\multicolumn{4}{@{}l}{\textbf{MBPP$\to$BCB; learning rate}} \\*
$\alpha=0.005$ & 43.751 & $+31.839$ & $[+27.964,+32.501]$ \\
$\alpha=0.01$ & 32.543 & $+20.631$ & $[+16.747,+21.306]$ \\
$\alpha=0.02$ & 20.546 & $+8.634$ & $[+5.374,+9.328]$ \\
$\alpha=0.04$ (primary) & 11.912 & $+0.000$ & --- \\
$\alpha=0.08$ & 9.801 & $-2.111$ & $[-4.066,-0.448]$ \\
$\alpha=0.16$ & 11.324 & $-0.588$ & $[-3.722,+1.729]$ \\
$\alpha=0.32$ & 16.209 & $+4.297$ & $[+0.088,+6.795]$ \\
\addlinespace
\multicolumn{4}{@{}l}{\textbf{MBPP$\to$BCB; band count}} \\*
$K=1$ & 11.891 & $-0.021$ & $[-0.058,+0.010]$ \\
$K=2$ & 11.902 & $-0.010$ & $[-0.043,+0.019]$ \\
$K=3$ (primary) & 11.912 & $+0.000$ & --- \\
$K=5$ & 12.101 & $+0.189$ & $[+0.029,+0.257]$ \\
$K=10$ & 14.815 & $+2.903$ & $[+1.626,+3.154]$ \\
$K=20$ & 18.330 & $+6.418$ & $[+4.333,+6.737]$ \\
\addlinespace
\multicolumn{4}{@{}l}{\textbf{MMLU-shift; learning rate}} \\*
$\alpha=0.005$ & 9.584 & $+6.755$ & $[+4.396,+7.753]$ \\
$\alpha=0.01$ & 6.979 & $+4.150$ & $[+2.185,+4.723]$ \\
$\alpha=0.02$ & 4.501 & $+1.672$ & $[+0.224,+2.174]$ \\
$\alpha=0.04$ (primary) & 2.829 & $+0.000$ & --- \\
$\alpha=0.08$ & 2.424 & $-0.405$ & $[-1.280,+0.638]$ \\
$\alpha=0.16$ & 3.508 & $+0.678$ & $[-0.843,+2.015]$ \\
$\alpha=0.32$ & 7.428 & $+4.599$ & $[+1.933,+6.200]$ \\
\addlinespace
\multicolumn{4}{@{}l}{\textbf{MMLU-shift; band count}} \\*
$K=1$ & 2.693 & $-0.136$ & $[-0.185,+0.100]$ \\
$K=2$ & 2.719 & $-0.110$ & $[-0.152,+0.077]$ \\
$K=3$ (primary) & 2.829 & $+0.000$ & --- \\
$K=5$ & 3.119 & $+0.289$ & $[-0.286,+0.502]$ \\
$K=10$ & 3.599 & $+0.770$ & $[-0.233,+1.617]$ \\
$K=20$ & 4.376 & $+1.547$ & $[-0.130,+2.393]$ \\
\addlinespace
\multicolumn{4}{@{}l}{\textbf{GSM8K$\to$MATH; learning rate}} \\*
$\alpha=0.005$ & 12.093 & $+8.685$ & $[+3.131,+10.001]$ \\
$\alpha=0.01$ & 7.476 & $+4.068$ & $[-1.359,+5.385]$ \\
$\alpha=0.02$ & 4.082 & $+0.674$ & $[-1.522,+1.989]$ \\
$\alpha=0.04$ (primary) & 3.408 & $+0.000$ & --- \\
$\alpha=0.08$ & 5.245 & $+1.837$ & $[-0.134,+2.064]$ \\
$\alpha=0.16$ & 8.800 & $+5.392$ & $[+2.752,+5.752]$ \\
$\alpha=0.32$ & 14.706 & $+11.298$ & $[+8.438,+11.833]$ \\
\addlinespace
\multicolumn{4}{@{}l}{\textbf{GSM8K$\to$MATH; band count}} \\*
$K=1$ & 3.405 & $-0.003$ & $[-0.022,+0.021]$ \\
$K=2$ & 3.405 & $-0.003$ & $[-0.022,+0.021]$ \\
$K=3$ (primary) & 3.408 & $+0.000$ & --- \\
$K=5$ & 3.404 & $-0.004$ & $[-0.070,+0.122]$ \\
$K=10$ & 3.615 & $+0.207$ & $[-0.330,+0.575]$ \\
$K=20$ & 3.859 & $+0.451$ & $[-0.647,+1.187]$ \\
\addlinespace
\multicolumn{4}{@{}l}{\textbf{HE$\to$HE+; learning rate}} \\*
$\alpha=0.005$ & 6.587 & $+2.998$ & $[-0.067,+3.976]$ \\
$\alpha=0.01$ & 5.117 & $+1.528$ & $[-0.799,+2.432]$ \\
$\alpha=0.02$ & 3.722 & $+0.134$ & $[-0.782,+1.065]$ \\
$\alpha=0.04$ (primary) & 3.588 & $+0.000$ & --- \\
$\alpha=0.08$ & 4.904 & $+1.316$ & $[+0.201,+1.463]$ \\
$\alpha=0.16$ & 7.011 & $+3.422$ & $[+1.769,+3.669]$ \\
$\alpha=0.32$ & 8.664 & $+5.076$ & $[+3.865,+5.477]$ \\
\addlinespace
\multicolumn{4}{@{}l}{\textbf{HE$\to$HE+; band count}} \\*
$K=1$ & 3.962 & $+0.374$ & $[-0.759,+0.955]$ \\
$K=2$ & 3.967 & $+0.379$ & $[-0.756,+0.958]$ \\
$K=3$ (primary) & 3.588 & $+0.000$ & --- \\
$K=5$ & 3.552 & $-0.037$ & $[-0.056,+0.020]$ \\
$K=10$ & 3.443 & $-0.146$ & $[-0.226,+0.144]$ \\
$K=20$ & 3.823 & $+0.235$ & $[-0.698,+0.706]$ \\
\addlinespace
\multicolumn{4}{@{}l}{\textbf{MBPP$\to$MBPP+; learning rate}} \\*
$\alpha=0.005$ & 7.175 & $+4.263$ & $[-1.081,+5.487]$ \\
$\alpha=0.01$ & 4.498 & $+1.585$ & $[-1.664,+2.793]$ \\
$\alpha=0.02$ & 2.656 & $-0.257$ & $[-1.130,+0.960]$ \\
$\alpha=0.04$ (primary) & 2.912 & $+0.000$ & --- \\
$\alpha=0.08$ & 5.101 & $+2.188$ & $[+0.463,+2.349]$ \\
$\alpha=0.16$ & 8.174 & $+5.262$ & $[+3.032,+5.559]$ \\
$\alpha=0.32$ & 13.300 & $+10.388$ & $[+7.830,+10.828]$ \\
\addlinespace
\multicolumn{4}{@{}l}{\textbf{MBPP$\to$MBPP+; band count}} \\*
$K=1$ & 3.456 & $+0.544$ & $[+0.272,+0.691]$ \\
$K=2$ & 3.452 & $+0.540$ & $[+0.260,+0.714]$ \\
$K=3$ (primary) & 2.912 & $+0.000$ & --- \\
$K=5$ & 2.904 & $-0.008$ & $[-0.095,+0.110]$ \\
$K=10$ & 2.830 & $-0.082$ & $[-0.536,+0.578]$ \\
$K=20$ & 2.958 & $+0.046$ & $[-0.911,+1.049]$ \\
\addlinespace
\multicolumn{4}{@{}l}{\textbf{MBPP$\to$MBPP+; shrinkage}} \\*
$k_s=0$ & 2.066 & $-0.846$ & $[-0.931,-0.142]$ \\
$k_s=5$ & 2.192 & $-0.720$ & $[-0.784,-0.177]$ \\
$k_s=10$ & 2.270 & $-0.642$ & $[-0.689,-0.174]$ \\
$k_s=20$ & 2.408 & $-0.504$ & $[-0.545,-0.198]$ \\
$k_s=50$ & 2.674 & $-0.238$ & $[-0.314,-0.140]$ \\
$k_s=100$ (primary) & 2.912 & $+0.000$ & --- \\
$k_s=200$ & 3.143 & $+0.230$ & $[+0.152,+0.401]$ \\
$k_s=500$ & 3.364 & $+0.452$ & $[+0.322,+0.677]$ \\
$k_s=1000$ & 3.463 & $+0.551$ & $[+0.410,+0.778]$ \\
\addlinespace
\multicolumn{4}{@{}l}{\textbf{MBPP$\to$MBPP+; success/failure rates}} \\*
$0.04/0.04$ (primary) & 2.912 & $+0.000$ & --- \\
$0.02/0.08$ & 27.411 & $+24.499$ & $[+17.914,+26.635]$ \\
$0.08/0.02$ & 13.388 & $+10.476$ & $[+4.806,+12.229]$ \\
$0.01/0.04$ & 24.899 & $+21.986$ & $[+15.391,+24.118]$ \\
$0.04/0.01$ & 13.594 & $+10.682$ & $[+5.070,+12.367]$ \\
$0.08/0.04$ & 8.860 & $+5.948$ & $[+0.463,+7.457]$ \\
\addlinespace
\bottomrule
\end{longtable}
\endgroup

\FloatBarrier

\subsection{Additional Sensitivities and Interpretation}\label{suppsec:other}

Table~\ref{supp:matched-cloud} reports the coding-target cold-start sensitivity on the same nine-cloud cohort used in the shift experiments. It preserves the four adverse cold-start comparisons and the inconclusive decayed-histogram comparison. The cold/warm sign difference is therefore not explained solely by changing the model cohort.

\begin{table}[!htbp]
\caption{Matched nine-cloud BCB cold-start sensitivity, without source warmup. Deltas are MARGIN minus comparator ECE. Four positive intervals support higher MARGIN ECE; the decayed-histogram comparison is inconclusive.}\label{supp:matched-cloud}
\centering
\small
\setlength{\tabcolsep}{3pt}
\begin{tabular}{@{}lrr@{}}
\toprule
Comparator & $\Delta$ & 95\% interval \\
\midrule
Decayed histogram & $-4.890$ & $[-8.576,+0.755]$ \\
Online Platt & $+3.232$ & $[+2.423,+4.012]$ \\
Sliding-window hist. & $+6.228$ & $[+2.290,+9.059]$ \\
Windowed acc. multiply & $+8.185$ & $[+3.094,+10.128]$ \\
Windowed acc. replace & $+7.427$ & $[+2.859,+9.810]$ \\
\bottomrule
\end{tabular}
\end{table}

\FloatBarrier

Table~\ref{supp:ties} gives the tie rates and the uniform-tie selection sensitivity for the cold-start coordination experiments of Section~\ref{sec:coordination}. With raw confidence, the highest-scoring answer groups are tied on most tasks. The fixed-order rule then selects the first tied group in model-identifier order, while the uniform rule takes the expected success over the tied groups. MARGIN's differences from raw confidence are positive under the uniform rule on all three benchmarks.

\begin{table}[!htbp]
\caption{Tie rates and uniform-tie selection sensitivity for the cold-start coordination experiments on the 18-model pool. Tie share is the percentage of tasks on which the highest-scoring answer groups are tied, averaged over histories. Selection accuracy takes the expected success over tied groups. Deltas are MARGIN minus Raw with paired 95\% intervals.}\label{supp:ties}
\centering
\small
\setlength{\tabcolsep}{3pt}
\begin{tabular}{@{}lrr@{\hspace{14pt}}rrrr@{}}
\toprule
 & \multicolumn{2}{c}{Tie share (\%)} & \multicolumn{4}{c}{Selection, uniform ties} \\
\cmidrule(lr){2-3}\cmidrule(l){4-7}
Benchmark & Raw & MARGIN & Raw & MARGIN & $\Delta$ & 95\% interval \\
\midrule
HumanEval & 62.2 & 16.1 & 86.20 & 94.98 & $+8.79$ & $[+6.62,+11.12]$ \\
MBPP & 59.9 & 3.2 & 81.57 & 92.17 & $+10.60$ & $[+8.43,+12.86]$ \\
BigCodeBench & 78.5 & 1.2 & 15.10 & 27.85 & $+12.75$ & $[+8.09,+17.39]$ \\
\bottomrule
\end{tabular}
\end{table}

\FloatBarrier

The retained extended record also includes matched confidence-channel, transfer, and delay analyses. They use distinct estimands from the main comparison. Agreement confidence is computed from five additional archived sample signatures matching the original candidate; correctness labels are not inputs to constructing that signal. Cross-channel comparisons retain the same original candidates and restrict support to those with both signals. A larger inference cost alone does not establish a scientific advantage for either channel.

The four supported transfer directions (HE to MBPP, MBPP to HE, and each source to BCB) hold learned source factors fixed on the target and compare them with from-scratch target learning. Those historical transfer summaries average per-model, per-history ECE rather than using the pooled prequential ECE of the main tables. Order variation describes repeated arrangements of the same observations and adds no new experiments. These auxiliary results do not establish broader drift coverage or revive historical-only benchmark claims.

\FloatBarrier

\end{appendices}


\begin{thebibliography}{39}
\providecommand{\natexlab}[1]{#1}
\providecommand{\url}[1]{\texttt{#1}}
\expandafter\ifx\csname urlstyle\endcsname\relax
  \providecommand{\doi}[1]{doi: #1}\else
  \providecommand{\doi}{doi: \begingroup \urlstyle{rm}\Url}\fi

\bibitem[Austin et~al.(2021)Austin, Odena, Nye, Bosma, Michalewski, Dohan,
  Jiang, Cai, Terry, Le, and Sutton]{austin2021program}
Jacob Austin, Augustus Odena, Maxwell Nye, Maarten Bosma, Henryk Michalewski,
  David Dohan, Ellen Jiang, Carrie Cai, Michael Terry, Quoc Le, and Charles
  Sutton.
\newblock Program synthesis with large language models.
\newblock arXiv preprint arXiv:2108.07732, 2021.

\bibitem[B\"{u}hlmann(1967)]{buhlmann1967credibility}
Hans B\"{u}hlmann.
\newblock Experience rating and credibility.
\newblock \emph{ASTIN Bulletin: The Journal of the International Actuarial
  Association}, 4\penalty0 (3):\penalty0 199--207, 1967.

\bibitem[Cesa-Bianchi and Lugosi(2006)]{cesabianchi2006prediction}
Nicol{\`o} Cesa-Bianchi and G{\'a}bor Lugosi.
\newblock \emph{Prediction, Learning, and Games}.
\newblock Cambridge University Press, 2006.

\bibitem[Chen et~al.(2024)Chen, Zaharia, and Zou]{chen2023frugalgpt}
Lingjiao Chen, Matei Zaharia, and James Zou.
\newblock {FrugalGPT}: How to use large language models while reducing cost and
  improving performance.
\newblock \emph{Transactions on Machine Learning Research}, 2024.
\newblock arXiv:2305.05176.

\bibitem[Chen et~al.(2021)Chen, Tworek, Jun, et~al.]{chen2021evaluating}
Mark Chen, Jerry Tworek, Heewoo Jun, et~al.
\newblock Evaluating large language models trained on code.
\newblock arXiv preprint arXiv:2107.03374, 2021.

\bibitem[Cobbe et~al.(2021)Cobbe, Kosaraju, Bavarian, Chen, Jun, Kaiser,
  Plappert, Tworek, Hilton, Nakano, Hesse, and Schulman]{cobbe2021verifiers}
Karl Cobbe, Vineet Kosaraju, Mohammad Bavarian, Mark Chen, Heewoo Jun, Lukasz
  Kaiser, Matthias Plappert, Jerry Tworek, Jacob Hilton, Reiichiro Nakano,
  Christopher Hesse, and John Schulman.
\newblock Training verifiers to solve math word problems.
\newblock arXiv preprint arXiv:2110.14168, 2021.

\bibitem[Dawid(1982)]{dawid1982calibrated}
A.~Philip Dawid.
\newblock The well-calibrated {B}ayesian.
\newblock \emph{Journal of the American Statistical Association}, 77\penalty0
  (379):\penalty0 605--610, 1982.

\bibitem[Du et~al.(2024)Du, Li, Torralba, Tenenbaum, and
  Mordatch]{du2023improving}
Yilun Du, Shuang Li, Antonio Torralba, Joshua~B. Tenenbaum, and Igor Mordatch.
\newblock Improving factuality and reasoning in language models through
  multiagent debate.
\newblock In \emph{Proceedings of the 41st International Conference on Machine
  Learning (ICML)}, 2024.

\bibitem[Farquhar et~al.(2024)Farquhar, Kossen, Kuhn, and
  Gal]{farquhar2024detecting}
Sebastian Farquhar, Jannik Kossen, Lorenz Kuhn, and Yarin Gal.
\newblock Detecting hallucinations in large language models using semantic
  entropy.
\newblock \emph{Nature}, 630:\penalty0 625--630, 2024.

\bibitem[Foster and Vohra(1998)]{foster1998asymptotic}
Dean~P. Foster and Rakesh~V. Vohra.
\newblock Asymptotic calibration.
\newblock \emph{Biometrika}, 85\penalty0 (2):\penalty0 379--390, 1998.

\bibitem[Geng et~al.(2024)Geng, Cai, Wang, Koeppl, Nakov, and
  Gurevych]{geng2024survey}
Jiahui Geng, Fengyu Cai, Yuxia Wang, Heinz Koeppl, Preslav Nakov, and Iryna
  Gurevych.
\newblock A survey of confidence estimation and calibration in large language
  models.
\newblock In \emph{Proceedings of the 2024 Conference of the North American
  Chapter of the Association for Computational Linguistics (NAACL)}, pages
  6577--6595, 2024.

\bibitem[Gerych et~al.(2024)Gerych, Rizk, Isahagian, Muthusamy, Duesterwald,
  and Venkateswaran]{gerych2024whoknows}
Walter Gerych, Yara Rizk, Vatche Isahagian, Vinod Muthusamy, Evelyn
  Duesterwald, and Praveen Venkateswaran.
\newblock Who knows the answer? finding the best model and prompt for each
  query using confidence-based search.
\newblock In \emph{Proceedings of the 38th AAAI Conference on Artificial
  Intelligence}, 2024.

\bibitem[Gibbs and Cand{\`e}s(2021)]{gibbs2021adaptive}
Isaac Gibbs and Emmanuel Cand{\`e}s.
\newblock Adaptive conformal inference under distribution shift.
\newblock In \emph{Advances in Neural Information Processing Systems
  (NeurIPS)}, 2021.

\bibitem[Guo et~al.(2017)Guo, Pleiss, Sun, and Weinberger]{guo2017calibration}
Chuan Guo, Geoff Pleiss, Yu~Sun, and Kilian~Q. Weinberger.
\newblock On calibration of modern neural networks.
\newblock In \emph{Proceedings of the 34th International Conference on Machine
  Learning (ICML)}, pages 1321--1330, 2017.

\bibitem[H\'{e}bert-Johnson et~al.(2018)H\'{e}bert-Johnson, Kim, Reingold, and
  Rothblum]{hebert2018multicalibration}
Ursula H\'{e}bert-Johnson, Michael~P. Kim, Omer Reingold, and Guy~N. Rothblum.
\newblock Multicalibration: Calibration for the (computationally-identifiable)
  masses.
\newblock In \emph{Proceedings of the 35th International Conference on Machine
  Learning (ICML)}, volume~80 of \emph{Proceedings of Machine Learning
  Research}, pages 1939--1948. PMLR, 2018.

\bibitem[Hendrycks et~al.(2021{\natexlab{a}})Hendrycks, Burns, Basart, Zou,
  Mazeika, Song, and Steinhardt]{hendrycks2021mmlu}
Dan Hendrycks, Collin Burns, Steven Basart, Andy Zou, Mantas Mazeika, Dawn
  Song, and Jacob Steinhardt.
\newblock Measuring massive multitask language understanding.
\newblock In \emph{Proceedings of the 9th International Conference on Learning
  Representations (ICLR)}, 2021{\natexlab{a}}.
\newblock arXiv:2009.03300.

\bibitem[Hendrycks et~al.(2021{\natexlab{b}})Hendrycks, Burns, Kadavath, Arora,
  Basart, Tang, Song, and Steinhardt]{hendrycks2021math}
Dan Hendrycks, Collin Burns, Saurav Kadavath, Akul Arora, Steven Basart, Eric
  Tang, Dawn Song, and Jacob Steinhardt.
\newblock Measuring mathematical problem solving with the {MATH} dataset.
\newblock In \emph{Proceedings of the Neural Information Processing Systems
  Track on Datasets and Benchmarks}, volume~1, 2021{\natexlab{b}}.
\newblock arXiv:2103.03874.

\bibitem[Hunter(1986)]{hunter1986ewma}
J.~Stuart Hunter.
\newblock The exponentially weighted moving average.
\newblock \emph{Journal of Quality Technology}, 18\penalty0 (4):\penalty0
  203--210, 1986.

\bibitem[J{\o}sang et~al.(2007)J{\o}sang, Ismail, and Boyd]{josang2007survey}
Audun J{\o}sang, Roslan Ismail, and Colin Boyd.
\newblock A survey of trust and reputation systems for online service
  provision.
\newblock \emph{Decision Support Systems}, 43\penalty0 (2):\penalty0 618--644,
  2007.

\bibitem[Kamvar et~al.(2003)Kamvar, Schlosser, and
  Garcia-Molina]{kamvar2003eigentrust}
Sepandar~D. Kamvar, Mario~T. Schlosser, and Hector Garcia-Molina.
\newblock The {EigenTrust} algorithm for reputation management in {P2P}
  networks.
\newblock In \emph{Proceedings of the 12th International Conference on World
  Wide Web (WWW)}, pages 640--651, 2003.

\bibitem[Kuhn et~al.(2023)Kuhn, Gal, and Farquhar]{kuhn2023semantic}
Lorenz Kuhn, Yarin Gal, and Sebastian Farquhar.
\newblock Semantic uncertainty: Linguistic invariances for uncertainty
  estimation in natural language generation.
\newblock In \emph{Proceedings of the 11th International Conference on Learning
  Representations (ICLR)}, 2023.

\bibitem[La~Malfa et~al.(2025)La~Malfa, La~Malfa, Marro, Zhang, Black, Luck,
  Torr, and Wooldridge]{lamalfa2025llms}
Emanuele La~Malfa, Gabriele La~Malfa, Samuele Marro, Jie~M. Zhang, Elizabeth
  Black, Michael Luck, Philip Torr, and Michael Wooldridge.
\newblock Large language models miss the multi-agent mark.
\newblock In \emph{Advances in Neural Information Processing Systems (NeurIPS),
  Position Track}, 2025.
\newblock arXiv:2505.21298.

\bibitem[Li et~al.(2025)Li, Xiong, Wu, and Hooi]{li2025conftuner}
Yibo Li, Miao Xiong, Jiaying Wu, and Bryan Hooi.
\newblock {ConfTuner}: Training large language models to express their
  confidence verbally.
\newblock In \emph{Advances in Neural Information Processing Systems},
  volume~38, pages 53484--53513, 2025.
\newblock \doi{10.52202/085713-1783}.

\bibitem[Liu et~al.(2023)Liu, Xia, Wang, and Zhang]{liu2023evalplus}
Jiawei Liu, Chunqiu~Steven Xia, Yuyao Wang, and Lingming Zhang.
\newblock Is your code generated by {ChatGPT} really correct? rigorous
  evaluation of large language models for code generation.
\newblock In \emph{Advances in Neural Information Processing Systems},
  volume~36, 2023.
\newblock \doi{10.52202/075280-0943}.

\bibitem[Liu et~al.(2025)Liu, Chen, Da, Chen, Lin, and Wei]{liu2025uq}
Xiaoou Liu, Tiejin Chen, Longchao Da, Chacha Chen, Zhen Lin, and Hua Wei.
\newblock Uncertainty quantification and confidence calibration in large
  language models: A survey.
\newblock In \emph{Proceedings of the 31st ACM SIGKDD Conference on Knowledge
  Discovery and Data Mining}, 2025.
\newblock arXiv:2503.15850.

\bibitem[Luo et~al.(2025)Luo, Wang, Li, and Wei]{daca2025}
Beier Luo, Shuoyuan Wang, Sharon Li, and Hongxin Wei.
\newblock Your pre-trained {LLM} is secretly an unsupervised confidence
  calibrator.
\newblock In \emph{Advances in Neural Information Processing Systems
  (NeurIPS)}, 2025.
\newblock arXiv:2505.16690.

\bibitem[Naeini et~al.(2015)Naeini, Cooper, and
  Hauskrecht]{naeini2015obtaining}
Mahdi~Pakdaman Naeini, Gregory~F. Cooper, and Milos Hauskrecht.
\newblock Obtaining well calibrated probabilities using {B}ayesian binning.
\newblock In \emph{Proceedings of the 29th AAAI Conference on Artificial
  Intelligence}, pages 2901--2907, 2015.

\bibitem[Ovadia et~al.(2019)Ovadia, Fertig, Ren, Nado, Sculley, Nowozin,
  Dillon, Lakshminarayanan, and Snoek]{ovadia2019trust}
Yaniv Ovadia, Emily Fertig, Jie Ren, Zachary Nado, D.~Sculley, Sebastian
  Nowozin, Joshua~V. Dillon, Balaji Lakshminarayanan, and Jasper Snoek.
\newblock Can you trust your model's uncertainty? evaluating predictive
  uncertainty under dataset shift.
\newblock In \emph{Advances in Neural Information Processing Systems
  (NeurIPS)}, 2019.

\bibitem[Platt(2000)]{platt1999probabilistic}
John~C. Platt.
\newblock Probabilistic outputs for support vector machines and comparisons to
  regularized likelihood methods.
\newblock In \emph{Advances in Large-Margin Classifiers}, pages 61--74. MIT
  Press, 2000.

\bibitem[Roberts(1959)]{roberts1959spc}
S.~W. Roberts.
\newblock Control chart tests based on geometric moving averages.
\newblock \emph{Technometrics}, 1\penalty0 (3):\penalty0 239--250, 1959.

\bibitem[Shen et~al.(2024)Shen, Das, Greenewald, Sattigeri, Wornell, and
  Ghosh]{shen2024thermometer}
Maohao Shen, Subhro Das, Kristjan Greenewald, Prasanna Sattigeri, Gregory
  Wornell, and Soumya Ghosh.
\newblock Thermometer: Towards universal calibration for large language models.
\newblock In \emph{Proceedings of the 41st International Conference on Machine
  Learning (ICML)}, 2024.

\bibitem[Smit et~al.(2024)Smit, Grinsztajn, Duckworth, Barrett, and
  Pretorius]{smit2024mad}
Andries~Petrus Smit, Nathan Grinsztajn, Paul Duckworth, Thomas~D. Barrett, and
  Arnu Pretorius.
\newblock Should we be going {MAD}? a look at multi-agent debate strategies for
  {LLMs}.
\newblock In \emph{Proceedings of the 41st International Conference on Machine
  Learning}, volume 235, pages 45883--45905. PMLR, 2024.

\bibitem[Tian et~al.(2023)Tian, Mitchell, Zhou, Sharma, Rafailov, Yao, Finn,
  and Manning]{tian2023just}
Katherine Tian, Eric Mitchell, Allan Zhou, Archit Sharma, Rafael Rafailov,
  Huaxiu Yao, Chelsea Finn, and Christopher~D. Manning.
\newblock Just ask for calibration: Strategies for eliciting calibrated
  confidence scores from language models fine-tuned with human feedback.
\newblock In \emph{Proceedings of the 2023 Conference on Empirical Methods in
  Natural Language Processing (EMNLP)}, 2023.
\newblock arXiv:2305.14975.

\bibitem[Wang et~al.(2023)Wang, Wei, Schuurmans, Le, Chi, Narang, Chowdhery,
  and Zhou]{wang2023selfconsistency}
Xuezhi Wang, Jason Wei, Dale Schuurmans, Quoc Le, Ed~Chi, Sharan Narang,
  Aakanksha Chowdhery, and Denny Zhou.
\newblock Self-consistency improves chain of thought reasoning in language
  models.
\newblock In \emph{Proceedings of the 11th International Conference on Learning
  Representations (ICLR)}, 2023.

\bibitem[Wang et~al.(2024)Wang, Ma, Zhang, et~al.]{wang2024mmlupro}
Yubo Wang, Xueguang Ma, Ge~Zhang, et~al.
\newblock {MMLU-Pro}: A more robust and challenging multi-task language
  understanding benchmark.
\newblock In \emph{Advances in Neural Information Processing Systems},
  volume~37, 2024.
\newblock \doi{10.52202/079017-3018}.

\bibitem[Wu et~al.(2024)Wu, Bansal, Zhang, Wu, Li, Zhu, Jiang, Zhang, Zhang,
  Liu, Awadallah, White, Burger, and Wang]{wu2023autogen}
Qingyun Wu, Gagan Bansal, Jieyu Zhang, Yiran Wu, Beibin Li, Erkang Zhu,
  Li~Jiang, Xiaoyun Zhang, Shaokun Zhang, Jiale Liu, Ahmed~Hassan Awadallah,
  Ryen~W. White, Doug Burger, and Chi Wang.
\newblock {AutoGen}: Enabling next-gen {LLM} applications via multi-agent
  conversation.
\newblock In \emph{COLM 2024}, 2024.
\newblock arXiv:2308.08155.

\bibitem[Xiong et~al.(2024)Xiong, Hu, Lu, Li, Fu, He, and Hooi]{xiong2024llm}
Miao Xiong, Zhiyuan Hu, Xinyang Lu, Yifei Li, Jie Fu, Junxian He, and Bryan
  Hooi.
\newblock Can {LLMs} express their uncertainty? an empirical evaluation of
  confidence elicitation in {LLMs}.
\newblock In \emph{Proceedings of the 12th International Conference on Learning
  Representations (ICLR)}, 2024.
\newblock arXiv:2306.13063.

\bibitem[Zhou and Chen(2025)]{zhou2025ahmad}
Yan Zhou and Yanguang Chen.
\newblock Adaptive heterogeneous multi-agent debate for enhanced educational
  and factual reasoning in large language models.
\newblock \emph{Journal of King Saud University -- Computer and Information
  Sciences}, 37\penalty0 (10):\penalty0 330, 2025.
\newblock \doi{10.1007/s44443-025-00353-3}.

\bibitem[Zhuo et~al.(2025)Zhuo, Vu, Chim, et~al.]{zhuo2025bigcodebench}
Terry~Yue Zhuo, Minh~Chien Vu, Jenny Chim, et~al.
\newblock {BigCodeBench}: Benchmarking code generation with diverse function
  calls and complex instructions.
\newblock In \emph{Proceedings of the 13th International Conference on Learning
  Representations (ICLR)}, 2025.
\newblock arXiv:2406.15877.

\end{thebibliography}
\end{document}